\documentclass[lettersize,journal]{IEEEtran}

\ifCLASSINFOpdf
  \usepackage{graphicx}
  \DeclareGraphicsExtensions{.pdf,.png,.jpg}
  \graphicspath{{figs/}}
\else
  \usepackage[dvips]{graphicx}
  \DeclareGraphicsExtensions{.eps}
  \graphicspath{{figs/}}
\fi

\usepackage{epstopdf}
\usepackage{cite}
\usepackage{subcaption}
\usepackage{amsmath,amssymb,amsfonts}
\usepackage{amsthm}
\usepackage{fancyhdr}
\usepackage{algorithmic}
\usepackage{bbm}
\usepackage{authblk}
\usepackage[ruled,linesnumbered,lined,boxed,commentsnumbered]{algorithm2e}
\usepackage{graphicx}
\usepackage{url}
\usepackage{textcomp}
\usepackage{enumitem}
\usepackage{caption}
\usepackage{pifont}
\usepackage{commath}
\usepackage{xcolor}
\usepackage{upgreek}
\usepackage{multirow}
\usepackage{mdframed}
\usepackage{booktabs} 
\usepackage{array}    
\usepackage{siunitx}  

\newtheorem{assumption}{Assumption}

\newtheorem{proposition}{Proposition}

\usepackage{bm}
\usepackage[left=0.63in,top=0.7in,right=0.63in,bottom=0.95in]{geometry}

\usepackage{cases}

\usepackage{afterpage}

\begin{document}
\title{Efficient Onboard Vision-Language Inference in UAV-Enabled Low-Altitude Economy Networks via LLM-Enhanced Optimization}

\author{Yang~Li, Ruichen~Zhang, Yinqiu~Liu, Guangyuan~Liu, Abbas~Jamalipour,~\IEEEmembership{Fellow,~IEEE,}
Xianbin~Wang,~\IEEEmembership{Fellow,~IEEE,}
and Dong In Kim,~\IEEEmembership{Fellow,~IEEE,}
\thanks{Y. Li and G. Liu are with the College of Computing and Data Science, the Energy Research Institute @ NTU, Interdisciplinary Graduate Program, Nanyang Technological University, Singapore (e-mail: yang048@e.ntu.edu.sg; liug0022@e.ntu.edu.sg).
} 
\thanks{R. Zhang, and Y. Liu are with the College
of Computing and Data Science, Nanyang Technological University,
Singapore (e-mails: ruichen.zhang@ntu.edu.sg; yinqiu001@e.ntu.edu.sg)}
\thanks{A. Jamalipour is with The University of Sydney, Sydney, Australia (e-mail: a.jamalipour@ieee.org).}
\thanks{X. Wang is with the Department of Electrical and
Computer Engineering, Western University, London, Canada (e-mail: xianbin.wang@uwo.ca).}
\thanks{D. I. Kim is with the Department of Electrical and Computer Engineering, Sungkyunkwan University, South Korea (email: dongin@skku.edu).}
}

\maketitle
\begin{abstract}
The rapid advancement of Low-Altitude Economy Networks (LAENets) has enabled a variety of applications, including aerial surveillance, environmental sensing, and semantic data collection. To support these scenarios, unmanned aerial vehicles (UAVs) equipped with onboard vision-language models (VLMs) offer a promising solution for real-time multimodal inference. However, ensuring both inference accuracy and communication efficiency remains a significant challenge due to limited onboard resources and dynamic network conditions. In this paper, we first propose a UAV-enabled LAENet system model that jointly captures UAV mobility, user-UAV communication, and the onboard visual question answering (VQA) pipeline. Based on this model, we formulate a mixed-integer non-convex optimization problem to minimize task latency and power consumption under user-specific accuracy constraints. To solve the problem, we design a hierarchical optimization framework composed of two parts: (i) an Alternating Resolution and Power Optimization (ARPO) algorithm for resource allocation under accuracy constraints, and (ii) a Large Language Model-augmented Reinforcement Learning Approach (LLaRA) for adaptive UAV trajectory optimization.
The large language model (LLM) serves as an expert in refining reward design  of reinforcement learning in an offline fashion, introducing no additional latency in real-time decision-making. Numerical results demonstrate the efficacy of our proposed framework in improving inference performance and communication efficiency under dynamic LAENet conditions.

\end{abstract}

\begin{IEEEkeywords}
Low-altitude economy networks, vision-language model, unmanned aerial vehicles, trajectory optimization, large language model, and deep reinforcement learning.
\end{IEEEkeywords}

\IEEEpeerreviewmaketitle

\section{Introduction}\label{sec:introduction}
Low-Altitude Economy Networks (LAENets) have recently garnered growing attention as a novel paradigm that leverages the low-altitude airspace (typically below 1000 meters) to deliver digital services~\cite{cai2025large}. Their primary goal is to unlock commercial and societal benefits through flexible aerial operations.
Specifically, a typical LAENet refers to an integrated network composed of low-altitude aerial platforms, such as general aviation aircraft, Unmanned Aerial Vehicles (UAVs), and electric Vertical Take-Off and Landing (eVTOL) aircraft. Compared to ground-based systems, these intelligent platforms offer unique advantages of high mobility and adaptive deployment, making LAENets suitable for diverse services, including wireless communication, environmental sensing, and edge computing~\cite{he2025satellite,cai2025secure,wang2025toward}.  Thereafter, 
LAENets are expected to play a significant role in supporting the pervasive services envisioned for future 6G networks. This potential is also reflected in the rapid growth of the LAE economy. For instance, the Civil Aviation Administration of China claims that the country's low-altitude market is expected to grow from \$70 billion in 2023 to \$200 billion by 2025, and reach \$480 billion by 2035~\cite{china2024lowaltitude}.

\textbf{UAV-enabled LAENets.} 
Among various LAENet platforms, UAV-based networks stand out due to their ability to perform intelligent tasks in complex low-altitude environments. 
Equipped with onboard sensors and processors, UAVs could function as flying agents that are capable of executing entire mission lifecycles autonomously without additional ground support. 
This autonomy also reduces reliance on the communication between UAV and base station, enabling service delivery in resource-constrained or infrastructure-sparse environments. 
As a result, recent works have highlighted the potential of UAV-enabled LAENets for different applications, including aerial surveillance~\cite{he2025ubiquitous}, disaster response~\cite{khan2022emerging}, and autonomous delivery~\cite{wang2025toward}. In addition, with advances in battery technology and AI-powered analytics, UAVs are promising to provide intelligent services with longer flight times, wider coverage, and enhanced analytical capabilities. These features align well with the service requirements of LAENets.

\textbf{Integrating VLMs into UAV-enabled LAENets.} 
Recent advances in vision-language models (VLMs), e.g., LLaVA~\cite{liu2023visual}, can strengthen UAV perception and reasoning, enabling applications such as object detection~\cite{yao2024vision} and geo-localization~\cite{wu2025clip}.
Hence, embedding these models onboard can offer a promising avenue to provide real-time, high-quality \textit{inference-as-a-service} to ground users within LAENets. Moreover, large VLMs have exhibited robust zero-shot generalization, eliminating the need for task-specific fine-tuning~\cite{li2023blip}. 
This makes them especially suitable for practical LAE scenarios, where ground users may request different types of inference on demand. By deploying a single VLM, UAVs can handle diverse tasks without switching models, reducing memory usage and improving overall efficiency. 
Notably, Zhao~\emph{et~al.}\cite{zhao2025general} demonstrated the onboard deployment of a 14B-parameter DeepSeek-R1~\cite{guo2025deepseek} on a medium-size UAV and achieved an inference speed of 5–6 tokens/sec for task planning. This provides a practical paradigm for integrating VLMs of comparable scale, such as LLaVA, into UAVs for aerial inference-as-a-service. Similar directions have also been explored in vehicular networks, where embodied AI frameworks integrate VLMs with reinforcement learning to enhance semantic communication and decision-making~\cite{11049053}.

\textbf{VLM Task-Driven Optimization in LAENets.}
Deploying VLMs on UAVs enables rich inference but strains communication and computing resources.
Effective resource allocation is thereafter essential to the utilization of limited resources in LAENets to improve system performance.
Unlike generic onboard computations, VLM inference couples perception quality with system performance: 
increasing input image resolution typically improves task accuracy but with diminishing returns that saturate beyond a task-dependent threshold, and it also incurs transmission overhead and lengthens model runtime~\cite{luo2024feast}.
This interplay of communication and computation performance as well as inference accuracy makes the optimization problem different from conventional UAV resource allocation, where objectives are typically confined to throughput or energy. 
Hence, a practical framework should explicitly model and optimize the accuracy–efficiency trade-off unique to VLM inference. 


\textbf{Challenges.} 
To summarize, while deploying VLMs on UAVs for agentic AI services offers great potential, it also introduces several major challenges for effective resource allocation:
\begin{enumerate}
    \item 
    Unlike traditional UAV tasks focusing on throughput or sensing coverage, VLM services require a holistic model that jointly captures image resolution selection, transmission delay, and model processing time, all of which jointly determine overall service performance. Such modeling is challenging as it must integrate communication, computation, and AI inference into a unified system model.
    \item 
    UAVs often operate under limits of flight regulations, communication resources, and onboard computing power. These constraints are critical when serving multiple users with diverse demands, such as different accuracy and latency requirements. How to formulate an optimization problem effectively managing the resources while ensuring service efficiency remains a key challenge.
    \item  
    Based on Challenges 1) and 2), the resulting optimization problem would inevitably involve heterogeneous variables, including discrete image resolutions, continuous transmit power, and dynamic UAV trajectories. Such a problem is typically mixed-integer, non-convex, and high-dimensional, which makes conventional optimization infeasible and necessitates new solution paradigms.
\end{enumerate}

\textbf{Our solution.}
To address those challenges, we formulate a joint optimization problem for LAENets that accounts for the accuracy–efficiency trade-off of VLM tasks.
Then, we develop a hierarchical framework to alternatively optimize the variables. 
Specifically, we first optimize image resolution and transmit power with standard solver-based methods. This is due to their efficiency in yielding the optimal solution for a finite mixed-integer subproblem whose continuous part (power allocation) is convex. We then optimize the UAV trajectory with deep reinforcement learning (DRL), since trajectory planning is a long-horizon problem that requires continuous optimization under dynamic channels and mobility constraints. Notably, we introduce the LLM as an \textit{offline reward designer} in our DRL setting. Traditionally, manually crafted reward functions in DRL are often heuristic, relying on ad hoc rules that struggle to adapt to dynamic system conditions. In contrast, the LLM leverages its reasoning and domain knowledge to automatically generate and iteratively refine reward functions that integrate multiple objectives in a principled manner. 
This automated design not only mitigates human bias but also has the potential to uncover latent optimization objectives, improving convergence speed and the overall effectiveness of DRL policies.

\textbf{Contribution.} The major contribution of this paper is listed as follows, each directly targeting one challenge:

\begin{itemize}
    \item \textbf{System Modeling in LAENet:}
     For Challenge 1), we propose a UAV-enabled LAENet that leverages VLMs for onboard inference. We then formulate a unified system model that captures the coupling between UAV trajectory, communication channels, and the onboard VLM inference pipeline, providing the foundations for holistic performance assessment.
    \item \textbf{Optimization Problem Formulation:}
    For Challenge 2), we formulate a mixed-integer problem that jointly optimizes UAV trajectory, user transmit power, and discrete image resolution to minimize the worst-case task latency while keeping power usage efficiency. 
    In addition, we empirically observe how image resolution affect: (i) task accuracy, (ii) model inference speed, and (iii) payload size, and convert observations into lookup tables. In the optimizer, these lookup tables encode the accuracy–efficiency trade-off and guide resource-allocation decisions, enabling the system to meet diverse user needs (e.g., different accuracy requirements) under limited system resources.
    \item \textbf{Hierarchical Optimization Framework:} For Challenge 3), we design a hierarchical framework that decouples the proposed problem into two subproblems for the solutions. The first subproblem on image resolution and power allocation is solved by an Alternating Resolution and Power Optimization (ARPO) algorithm. The second subproblem on UAV trajectory is tackled by an LLM-augmented Reinforcement learning Approach (LLaRA). Specifically, the LLM serves as an offline reward design expert to improve DRL for better convergence and more stable trajectory policies.
\end{itemize}

The remainder of the paper is structured as follows: Section~\ref{sec:Literature_review} reviews the related works; Section~\ref{sec:system_model} introduces the system model and problem formulation; Sections~\ref{sec:ARPO} and \ref{sec:LLaRA} detail the proposed ARPO and LLaRA methods. The complexity analysis is in Section~\ref{sec:complexity}. Simulation results are in Section~\ref{sec:experiments}. Finally, Section~\ref{sec:conclusion} concludes the whole paper.

\section{Related Work}\label{sec:Literature_review}
This section surveys recent work: Section~\ref{subsec:LAE} on LAENets, Section~\ref{subsec:VLM_edge} on VLMs for edge inference, and Section~\ref{subsec:LLM_DRL} on LLM-enhanced DRL methods.

\subsection{LAE Networks}\label{subsec:LAE}
In recent years, both the conception and core technologies of LAENets have been widely discussed in a series of studies~\cite{he2025satellite,wang2025toward,jiang20236g}.
For instance, He~\emph{et al.}\cite{he2025satellite} investigated the utilization of satellite technology for providing ubiquitous connectivity and enhancing communication, control, and computation in LAENets through advanced architectures and optimization schemes. 
Wang~\emph{et al.}\cite{wang2025toward} highlighted the synergy of communication, sensing, computing, and control technologies as a key driver for advancing LAE networks. Jiang~\emph{et al.}\cite{jiang20236g} reviewed the use of sensing‑communication integration and discussed prerequisite technologies, e.g., network coverage and aircraft detection, enhancing awareness in LAENets. 

In addition to conceptual discussions, recent studies also develop system‑level solutions to improve the efficiency of LAENets. These solutions focus on diverse technical perspectives. Some works enhance communication efficiency through \textit{advanced wireless technologies and optimizations}~\cite{ahmed2025toward,salim2025energy}. For instance, Ahmed~\emph{et al.}~\cite{ahmed2025toward} studied reconfigurable intelligent surface (RIS)-assisted UAV networks and detailed strategies such as trajectory optimization and power control to enhance energy efficiency in low-altitude operations. Likewise, Salim~\emph{et al.}~\cite{salim2025energy} proposed a DRL‑based method for energy optimization in irregular RIS‑aided UAV‑assisted networks. 
Some other works manage to integrate \textit{multi‑domain intelligence} with LAENets to enhance sensing and control efficiency: Yang~\emph{et al.}~\cite{yang2024embodied} proposed an embodied artificial intelligence (EAI) framework that unifies sensing, communication, computation, and control to enhance LAE efficiency. Besides, some studies address \textit{resource‑aware lifecycle management} for sustainable LAE growth, e.g., Zhou~\emph{et al.}~\cite{zhou2025unmanned} evaluated the long‑term operational and environmental impacts of UAV‑enabled services.

Despite these advances, LAENets still remain constrained in resources, particularly as they are expected to support more intelligent tasks.capabilities~\cite{10679152}.
These observations motivate our design of a more efficient LAENet.

\subsection{VLMs for Edge Inference}\label{subsec:VLM_edge}

Advanced VLMs have excelled in multimodal reasoning tasks. 
Hence, embedding VLMs into UAV has motivated a number of vision-language applications for LAENets, e.g., geo-localization~\cite{wu2025clip}, urban patrols~\cite{yuan2024patrol} and navigation~\cite{krupavs2025multimodal}. Recent vision-language-action architectures further extend these abilities to embodied and robotic control, enabling real-time inference on autonomous platforms.
We next examine optimization strategies for VLM‑based edge inference.

Some works focus on \textit{lightweight architectures and inference optimizations} to improve the efficiency of VLM deployment on edge platforms~\cite{sharshar2025vision,huang2025litevlm,luo2024feast}.
Sharshar~\emph{et al.}~\cite{sharshar2025vision} comprehensively surveyed popular strategies for making VLMs edge-compatible, including pruning, quantization, knowledge distillation, and hardware acceleration. LiteVLM~\cite{huang2025litevlm} presents an efficient VLM pipeline that leverages patch selection, token filtering, and speculative decoding techniques. It can achieve a $2.5\times$ reduction in latency without compromising task accuracy. Luo~\emph{et al.}~\cite{luo2024feast} proposed a mixture-of-resolution adaptation strategy with a dual-path design that preserves high-resolution benefits while improving efficiency, achieving nearly $3\times$ faster inference than LLaVA-1.5. Such efficiency-oriented frameworks offer valuable benchmarks for edge inference.

Moreover, some other studies investigate VLM‑based edge inference from a \textit{system‑level perspective}, focusing on adaptive orchestration and QoE optimization~\cite{sun2025disco,yang2024perllm,li2025distributed}.
For instance, Sun~\emph{et al.}~\cite{sun2025disco} proposed DiSCo, a device–server cooperative scheduler that dynamically routes LLM inference between local devices and servers to jointly optimize latency and energy; and such an idea can be readily extended to VLM scenarios.
PerLLM~\cite{yang2024perllm} presents a personalized scheduling framework that employs edge–cloud collaboration and an upper confidence bound algorithm to balance latency, energy, and service quality. Its adaptive mechanism is readily transferable to QoE‑aware VLM scheduling.
Li~\emph{et al.}~\cite{li2025distributed} proposed a distributed architecture for VLMs that addresses high computational demands by partitioning model components between edge devices and central servers. Specifically, vision components run on edge devices, while language generation runs on servers, resulting in up to a $33\%$ improvement in throughput.

Although our work does not focus on optimizing VLMs for edge inference, our LLaRA approach is significantly different to the above model‑level and system‑level solutions and can integrate them to further enhance aerial service efficiency.

\subsection{LLM-Enhanced DRL Methods}\label{subsec:LLM_DRL}

Despite its remarkable successes in various fields such as robotics, gaming, and autonomous control~\cite{11049053,10032267}, DRL still suffers from key issues including sample inefficiency, reward design difficulty, and limited language understanding~\cite{cao2024survey}. The recent emergence of advanced LLMs offers a promising way to address these issues with pre-trained knowledge and high-level general abilities. Next, we provide further details on how to leverage LLMs to enhance DRL.

To integrate LLMs into DRL, Cao~\emph{et al.}~\cite{cao2024survey} proposed a structured taxonomy categorizing LLMs into four complementary roles: \textit{information processors}, \textit{reward designers}, \textit{decision‑makers}, and \textit{generators}. Firstly, as \textit{information processors}, LLMs help extract meaningful features for downstream networks~\cite{pang2023natural} or translate natural information into formal task languages~\cite{spiegel2024informing}. Pang~\emph{et al.}~\cite{pang2023natural} proposed an inside‑out approach that trains an LLM to translate natural language instructions into task‑specific representations. Spiegel~\emph{et al.}~\cite{spiegel2024informing} developed RLang to convert natural language into Markov decision process (MDP) specifications for using prior knowledge. 

As \textit{reward designers}, LLMs leverage pre‑trained knowledge and code generation capabilities to provide implicit or explicit reward functions~\cite{kwon2023reward,ma2023eureka,xie2023text2reward}. Kwon~\emph{et al.}~\cite{kwon2023reward} streamlined reward design by prompting an LLM to act as a proxy reward function using examples and descriptions of desired behaviors. Eureka~\cite{ma2023eureka} introduces a self‑reflective algorithm to iteratively generate and refine reward functions via a coding LLM, which can achieve human‑level reward design and enable dexterous manipulation tasks. Another work, Text2Reward~\cite{xie2023text2reward}, enables the generation of shaped, dense reward functions as executable programs grounded in compact environment representations. 

As \textit{decision‑makers}, LLMs help action decisions, or provide action candidates and reference policies that guide exploration~\cite{li2022pre,shek2025option}.
Li~\emph{et al.}~\cite{li2022pre} used LLMs to combine goals and observations into sequential inputs, which improves both combinatorial generalization and out-of-distribution performance. In contrast, Shek~\emph{et al.}~\cite{shek2025option} introduced a hierarchical DRL framework that leverages LLMs to generate subgoals from task descriptions, select reusable options, and execute action-level policies, thereby improving decision-making. 

Finally, as \textit{generators}, LLMs simulate environment dynamics for RL~\cite{robine2023transformer} and provide interpretable policy explanations~\cite{silva2024towards}.
Robine~\emph{et al.}~\cite{robine2023transformer} used a Transformer‑based world model to address long‑term dependencies and achieve good performance on the Atari 100k benchmark. 
Silva~\emph{et al.}~\cite{silva2024towards} developed an adaptive explainability framework that personalizes explanation modalities to balance user preferences and task performance, enhancing decision-making in human-AI collaboration.

Prior studies on LLMs as reward designers~\cite{kwon2023reward,ma2023eureka,xie2023text2reward} mostly target single-task robot manipulation or game control and focus less on multi-user service scenarios. In contrast, we consider an LAENet where the UAV serves multiple users concurrently. The employed LLM needs to synthesize a compact reward that aggregates multi-user information, which is more challenging.

\section{System Model}\label{sec:system_model}
This section presents an overall system model for a LAENet. We first present the task formulation model in Section~\ref{subsec:Task}, and then model the UAV trajectory, user-UAV communication, and VLM inference in Sections~\ref{subsec:UAV_Trajectory}, \ref{subsec:communication}, and \ref{subsec:inference}, respectively. 
The optimization problem is formulated in Section~\ref{subsec:problem}.

\subsection{Task Formulation}\label{subsec:Task}
We consider a UAV-assisted LAE network comprising a single UAV that serves as an intelligent aerial agent and a set of ground users denoted by $\mathcal{N}: = \{1, \ldots, N\}$. 
Our system model be readily extended to multi-UAV scenarios by assigning different UAVs to serve diverse user sets, as in~\cite{zhan2019completion}.
We assume that ground user $n$ generates a visual-language task, and the UAV equipped with onboard computation capabilities executes those inference tasks in real time. To analyze user-UAV interactions over time, we model the system as a time-slotted network operating over a finite horizon $\mathcal{T}: = \{1, \ldots, T \}$.

Fig.~\ref{fig:system_model} illustrates our system model and its working pipelines. 
Ground users need to upload their requests to the UAV for TextVQA inference services~\cite {singh2019towards}. In detail, user $n$ generates one or more queries, where each query includes a visual input (i.e., an image $\mathbf{I}_n$ with a resolution of $r_n$) and a corresponding textual prompt $\mathbf{Q}_n$. 
Specifically, the resolution $r_n$ is defined as the total number of pixels, i.e., $r_n = H_n \times W_n$, where $H_n$ and $W_n$ denote the vertical and horizontal dimensions of $\mathbf{I}_n$, respectively.
The UAV must first collect data via the corresponding wireless communication, then perform the inference using its deployed VLM, and finally return the response. 

\begin{figure}
    \centering
    \includegraphics[width =\linewidth]{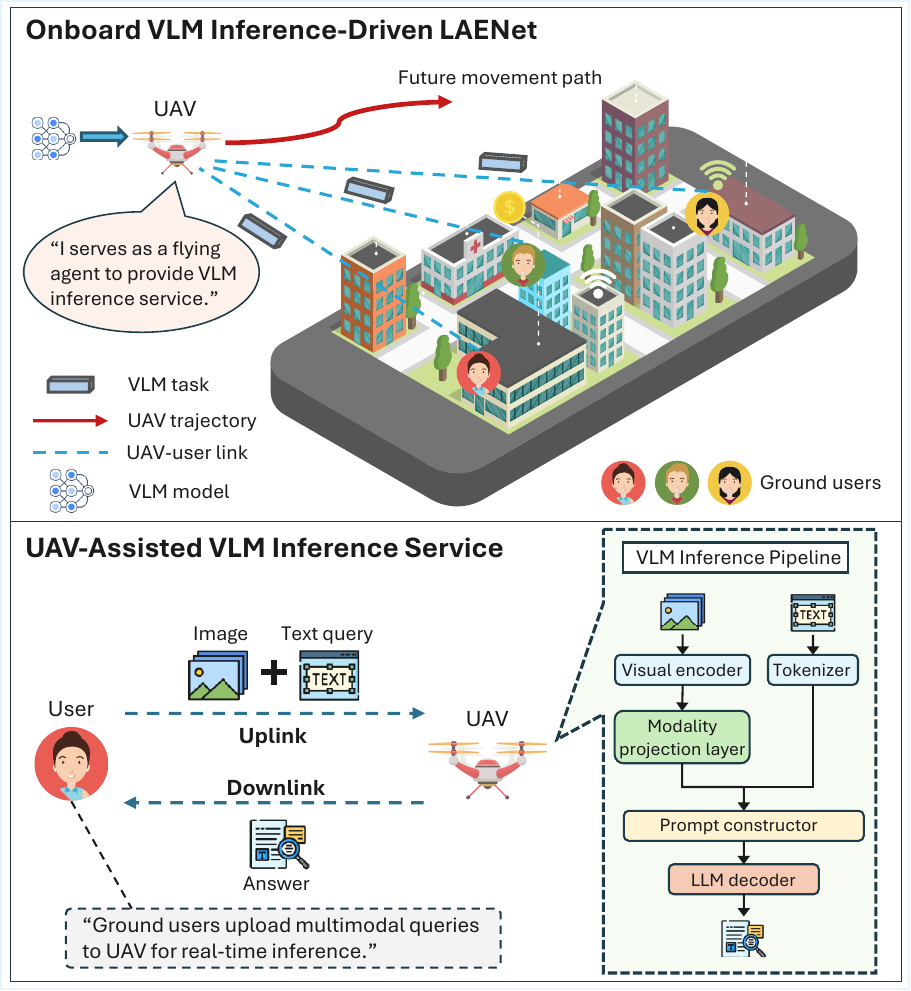}
    \caption{An overview of the onboard VLM inference-driven LAENet. The upper part depicts a UAV serving as a flying agent that providing VLM inference services to ground users; the lower part details the onboard VLM pipeline from user queries to answer generation.}
    \label{fig:system_model}
\end{figure}


\subsection{UAV Trajectory}\label{subsec:UAV_Trajectory}

We consider a 3D Cartesian coordinate system and denote the UAV’s horizontal position at time slot $t$ as $(x[t], y[t], z[t])$~\cite{zhan2020joint}. We further define $\mathbf{q}[t]:=(x[t], y[t])$ to represent the horizontal coordinate.
Let $\alpha$ denote the elemental time slot length, which is deemed to be sufficiently short such that the distances between the UAV and users remain approximately constant within each slot $t$.
The UAV's flight trajectory is discretized and represented in the set $\{(\mathbf{q}[t],z[t])\mid t \in \{ 1, \ldots, T \}\}$, with the continuous path approximated by connecting these discrete waypoints via line segments. To comply with aerial regulations, the UAV is required to operate within a specified altitude range. Hence, the UAV’s altitude must satisfy the following constraint:
\begin{align}
h^{\min} \leq z[t] \leq h^{\max}, \quad \forall t.
\end{align}
Besides, we assume that the UAV can independently control its horizontal and vertical flight speeds. Let $V^{\max}_{xy}$ and $V^{\max}_z$ denote the maximum allowable horizontal and vertical speeds, respectively. Accordingly, the UAV’s mobility is subject to the following constraints:
\begin{subequations}
\begin{numcases}{}
\|\mathbf{q}[t+1] - \mathbf{q}[t]\| \leq \alpha \cdot V_{xy}^{\max}, \quad \forall t,\\
|z[t+1] - z[t]| \leq \alpha \cdot V_{z}^{\max}, \quad \forall t.
\end{numcases}
\end{subequations}

\subsection{User-UAV Communication Model}\label{subsec:communication}

We focus on modeling uplink communication from ground users to the UAV, as it accounts for the majority of transmission costs due to the large image payload.
Similarly, we denote the position of each ground user $n\in \mathcal{N}$ as $(\mathbf{w}_n, h_n)$, where $\mathbf{w}_n:=(x_n,y_n)$ represents the corresponding horizontal coordinate. Next, we present channel modeling between ground users and the UAV, and then we analyze the transmission latency. 

\textbf{Channel Modeling.} 
We consider a quasi-static air-to-ground (A2G) channel, where the small-scale fading remains constant within each slot and may vary across slots.
Firstly, the UAV-user $n$ distance and elevation angle at slot $t$ is defined as follows:
\begin{subequations}
\begin{numcases}{}
d_n[t] = \sqrt{\|\mathbf q[t]-\mathbf w_n\|^2 + (z[t]-h_n)^2},\\
\theta_n[t] = \arctan\!\Big(\frac{z[t]-h_n}{\|\mathbf q[t]-\mathbf w_n\|}\Big).
\end{numcases}
\end{subequations}
Specifically, A2G has an elevation-dependent Line-of-Sight (LoS) probability. We adopt a probabilistic LoS/NLoS model:
\begin{subequations}
\begin{numcases}{}
P_{\mathrm{LoS}}(\theta_n[t]) = \frac{1}{1 + a \exp\!\big(-b(\theta_n[t]-a)\big)},\\
P_{\mathrm{NLoS}}(\theta_n[t]) = 1 - P_{\mathrm{LoS}}(\theta_n[t]),
\end{numcases}
\end{subequations}
where $a,b>0$ are LoS-probability parameters.
The large-scale gains under LoS and NLoS are
$\beta^{\mathrm{LoS}}_n[t] = \frac{\beta_0}{d_n[t]^{\gamma_{\mathrm{LoS}}}}$ and $
\beta^{\mathrm{NLoS}}_n[t] = \frac{\beta_0}{d_n[t]^{\gamma_{\mathrm{NLoS}}}}$, with $\beta_0$ the reference gain at $d_0=1$ m and path-loss exponents $\gamma_{\mathrm{LoS}},\gamma_{\mathrm{NLoS}}$.
Then, we define the elevation-aware gain by averaging over the LoS state:
\begin{align}
\bar\beta_n[t] \!=\! P_{\mathrm{LoS}}(\theta_n[t])\,\beta^{\mathrm{LoS}}_n[t]
\!+\! \big(1-P_{\mathrm{LoS}}(\theta_n[t])\big)\,\beta^{\mathrm{NLoS}}_n[t].
\end{align}
The baseband-equivalent channel coefficient is modeled as:
\begin{align}
h_n[t] = \sqrt{\bar\beta_n[t]}\,\hat h_n[t],\label{equa:channel_gain}
\end{align}
where $\hat{h}_{n}[t]$ denotes small-scale fading, modeled as a complex-valued random variable with zero mean and unit variance.

\textbf{Uplink Transmission Time.} 
Recall that each query from ground user $n$ includes an image $\mathbf{I}_n$ with resolution $r_n$ and a text query $\mathbf{Q}_n$. The total data size for transmission is denoted by $D_n(r_n)$, where the image dominates the payload size and the query size is negligible~\cite{singh2019towards}. Hence, we can hold that:
\begin{align}
D_n(r_n) = \big( D_n^{\text{img}}(r_n) + D_n^{\text{txt}}\big) \cdot M_n \approx D_n^{\text{img}}(r_n)\cdot M_n.\label{equa:data_size}
\end{align}
where $D_n^{\text{img}}(r_n)$ and $D_n^{\text{txt}}$ represent the transmitted data size for the image and query, respectively, and $M_n$ denotes the number of queries from ground user $n$.

Given channel gain $h_n[t]$ in (\ref{equa:channel_gain}) and let $P_n$ denote the transmit power of user $n$. The signal-to-noise ratio (SNR) at time step $t$ is expressed as:
\begin{align}
\text{SNR}_n[t]=\frac{P_{n}|h_n[t]|^2}{\sigma^2_n}, \quad \forall n,~t,\label{equa:SNR}
\end{align}
where $\sigma^2_n$ represents the noise power. Hence, the achievable uplink transmission rate is given by:
\begin{align}
    R_n[t]=B_n\log_2(1+\text{SNR}_n[t]),\quad \forall n,~t,\label{equa:rate}
\end{align}
where $B_n$ denotes the available bandwidth to each user $n$. Notably, $r_n$ and $P_n$ in our uplink communication model remain constant over $t$. We adopt the following assumption for simplicity and practicality:

\begin{assumption}
For each user $n$, the image resolution $r_n$ and transmit power $P_n$ remain fixed during its upload window, as resolution is set before transmission and power control typically operates on coarser timescales~\cite{fodor2010near}. 
In our setting, both the uplink window and the slot length are short, making slot-varying power control costly.
Changing $r_n$ or $P_n$ mid-stream also incurs re-encoding/control overhead and breaks a stable accuracy–latency mapping, offering limited gain~\cite{yang2015optimum}.
We therefore optimize $(r_n,P_n)$ at the session level and let the trajectory policy handle slot-level fluctuations~\cite{ao2023energy}.
\label{assumption:1}
\end{assumption}



We now analyze the uplink transmission time. Since the uplink rate $R_n[t]$ in~\eqref{equa:rate} can vary across slots, the simple form $D_n(r_n)/R_n$ does not generally apply.
With slot length $\alpha$, user $n$ can transmit at most $\alpha \cdot R_n[t]$ bits of data in slot $t$.
Given the payload size $D_n(r_n)$, we define the \emph{uplink completion index} (the smallest slot index by which user $n$ finishes uploading):
\begin{align}
T_n^{\text{cmp}}
= \min\Big\{ t\in\mathcal{T} \mid\sum_{\tau=1}^{t} \alpha\,R_n[\tau] \;\ge\; D_n(r_n) \Big\}.
\end{align}
The integer-slot latency is $\alpha \cdot T_n^{\text{cmp}}$. If the final slot is only partially used, the overall uplink time $T_n^{\text{up}}$ is computed as:
\begin{align}
T_n^{\text{up}}
= \alpha\big(T_n^{\text{cmp}}-1\big)
+ \frac{ D_n(r_n) - \sum_{\tau=1}^{T_n^{\text{cmp}}-1} \alpha\,R_n[\tau] }{ R_n[T_n^{\text{cmp}}] }, 
\end{align}
where $T_n^{\text{up}}$ is a continuous-time and does not need to align with slot boundaries.

\subsection{VLM Inference Model}\label{subsec:inference}

To support inference services for ground users, the VLM deployed on the UAV processes user inputs. First, we present an overview of the unified inference pipeline commonly adopted by modern VLMs~\cite{liu2023visual,li2023blip}. 

\begin{enumerate}
    \item \textbf{Visual Encoding:} 
    Given an image $\mathbf{I}_n(r_n)\!\in\!\mathbb{R}^{H_n\times W_n\times C}$\footnote{
    Here $C$ is the channel count (e.g., $C{=}3$ for RGB), and resolution $r_n=H_n\cdot W_n$. Most VLMs use a fixed input aspect ratio (e.g., $H_n{:}W_n=1{:}1$ in LLaVA); thus, given $r_n$, the pair $(H_n,W_n)$ is uniquely determined.}, a pretrained visual encoder extracts visual embeddings:
    \begin{align}
        \mathbf{Z}_n(r_n) = g\big(\tilde{\mathbf{I}}_n(r_n)\big),
    \end{align}
    where $\mathbf{Z}_n(r_n) \in \mathbb{R}^{L(r_n) \times d}$ is the obtained visual tokens; $L(r_n)$ is the number of tokens and $d$ is the embedding size. In common image encoders, $L(r_n)$ grows with $r_n$ for preserving richer image information~\cite{liu2023visual,luo2024feast}.
    \item \textbf{Modality Projection:}  
    The visual tokens are then mapped into the LLM input space via a learnable projector:
    \begin{align}
        \mathbf{E}_n(r_n) = \mathbf{W} \cdot \mathbf{Z}_n(r_n),
    \end{align}
    where $\mathbf{E}_n(r_n) \in \mathbb{R}^{L \times d'}$ and $d'$ matches the LLM’s input embedding size.
    \item \textbf{Prompt Construction:}  
    Next, the projected tokens are concatenated with the question tokens $\mathbf{Q}_n = \{w_1, \dots, w_T\}$ to form a multimodal prompt:
    \begin{align}
        \mathcal{X}_n(r_n) = [\mathbf{E}_n(r_n);\, \mathbf{Q}_n;\, \texttt{<eos>}],
    \end{align}
    where \texttt{<eos>} indicates the end of the question prompt.
    \item \textbf{Answer Generation:}  
    The prompt is finally passed into a decoder-only language model parameterized by $\theta$ to generate the predicted answer:
    \begin{align}
        \mathbf{A}_n^{\text{pred}}(r_n) = \mathrm{LLM}(\mathcal{X}_n(r_n); \theta),
    \end{align}
    where $\mathbf{A}_n^{\text{pred}}(r_n)$ denotes the generated answer sequence comprising $S$ tokens from the model vocabulary.
\end{enumerate}

\textbf{Inference Accuracy.}  
We consider a standard \textit{top-1 accuracy} metric from the TextVQA benchmark~\cite{singh2019towards} to evaluate the VLM inference performance:
\begin{align}
A_n(r_n)
= \mathbb{E}\biggl[\min\!\big\{\frac{1}{3}\sum_{k=1}^{K}\mathbb{I}\Big(\mathbf{A}_n^{\text{pred}}(r_n)=a_n^{\text{gt}(k)}\Big),\ 1\big\} \biggr],
\end{align}
where $\{a_n^{\text{gt}(k)}\}_{k=1}^{K}$ are the human-annotated ground truths and $\mathbf{A}_n^{\text{pred}}(r_n)$ is the model output under $r_n$. Larger $r_n$ often tends to preserve more fine-grained details of the image.
Empirically, $A_n(r_n)$ increases with resolution $r_n$, but exhibits diminishing returns, with smaller gains at high $r_n$. Such observations are also reported in~\cite{guo2024llava,luo2024feast}.

\textbf{Inference Latency.}
The selection of image resolutions also affects inference speed~\cite{liu2024improved}. Since the generated answer length is roughly resolution-invariant (e.g., short phrases in TextVQA), we define the processing time $T_n^{\text{proc}}(r_n)$ as:
\begin{align}
T_n^{\text{proc}}(r_n) = \frac{\mathbb{E}[|\mathbf{A}_n^{\text{pred}}|]}{v(r_n)},
\end{align}
where $\mathbb{E}[|\mathbf{A}_n^{\text{pred}}|]$ denotes the expected number of output tokens generated by the VLM, and $v(r_n)$ is the resolution-dependent inference speed (in tokens/s) measured empirically. The specific form of $v(r_n)$ is given in Section~\ref{subsec:accuracy}.


\textbf{Downlink Time.}
After inference, the generated answer $\mathbf{A}_n^{\text{pred}}(r_n)$ is sent back to user $n$. 
Since $\mathbf{A}_n^{\text{pred}}(r_n)$ is typically a short text (e.g., less than 20 tokens in TextVQA~\cite{singh2019towards}), we treat downlink latency as a fixed constant $T_n^{\text{down}}$.


\subsection{Problem Formulation}\label{subsec:problem}

Our first objective is to minimize the maximum latency for all ground users completing their tasks. Hence, we first express the total time consumption for each user $n$ as:
\begin{align}
T^{\text{total}}_n=T_n^{\text{up}}+T_n^{\text{proc}}(r_n)+T_n^{\text{down}},\quad \forall n,
\end{align}
and the first part of the objective function is expressed as:
\begin{align}
    \min \Big\{\max_{n \in \mathcal{N}}~T^{\text{total}}_n\Big\}.
\end{align}
To balance latency and power consumption, we add a reward term that encourages lower transmit power:
\begin{align}
    \max \Big\{-\sum_{n\in\mathcal{N}} P_n\Big\}.
\end{align}
To ensure inference performance, we further impose that the expected accuracy $A_n(r_n)$ meets each user’s minimum requirement $A_n^{\text{min}}$. We then formulate a joint optimization problem over the UAV trajectory $(\mathbf{q}[t],z[t])$, user transmit power $P_n$, and image resolution $r_n$, subject to constraints on UAV mobility, power, and resolution selection:
\begin{subequations}\label{Original_Problem}
\begin{align}
\mathbb{P}_0:&\min_{\substack{\{\mathbf{q}[t],z[t],\mathbf{P},\mathbf{r}\}}} \left\{\max_{n \in \mathcal{N}}~T^{\text{total}}_n+\zeta\sum_{n \in \mathcal{N}}P_n\right\}\tag{\ref{Original_Problem}} \\   
\text{s.t.}~
&\text{C1 :~}A_n(r_n) \geq A_n^{\min}, \quad \forall n,\label{constraint:accuracy}\\
&\text{C2 :~}h^{\min} \leq z[t] \leq h^{\max}, \quad \forall t,\label{constraint:altitude}\\
&\text{C3 :~}\|\mathbf{q}[t+1] - \mathbf{q}[t]\| \leq \alpha \cdot V_{xy}^{\max}, \quad \forall t,\label{constraint:speed_xy}\\
&\text{C4 :~}|z[t+1] - z[t]| \leq \alpha \cdot V_{z}^{\max}, \quad \forall t,\label{constraint:speed_z}\\
&\text{C5 :~} P_n \leq P_n^{\max}, \quad \forall n,\label{constraint:power}\\
&\text{C6 :~}r_n \in \mathcal{R}^{\mathrm{res}}= \{r^{(1)},\ldots, r^{(J)}\}, \quad \forall n.\label{constraint:resolution}
\end{align}
\end{subequations}
Here, $\zeta \geq 0$ is a tunable coefficient that controls the relative importance of power consumption versus latency minimization.
The constraint~\eqref{constraint:accuracy}) guarantees that the aerial inference accuracy remains above a required threshold $A_n^{\text{min}}$. Constraints (\ref{constraint:speed_xy}) and (\ref{constraint:speed_z}) limit the UAV's displacement between consecutive time slots.
Constraint~\eqref{constraint:power} restricts each user's transmit power within feasible bounds. 
Finally, constraint~\eqref{constraint:resolution} restricts the resolution selection to a finite candidate set $\mathcal{R}^{\text{res}} = \{r^{(1)},\ldots, r^{(J)}\}$ containing $J$ supported resolution choices.

Solving problem $\mathbb{P}_0$ is challenging for three reasons. First, it is a mixed-integer non-linear program (MINLP) involving coupled discrete (e.g., resolution) and continuous (e.g., trajectory) variables, making it \textit{NP-hard}. Second, it is non-convex due to the dependence of SNR on $(\mathbf{q}[t],z[t])$ and $P_n$, with resolution $r_n$ further complicating uplink latency. Third, its time-sequential nature enlarges the solution space with $T$, causing high computational burden and making convex or heuristic methods infeasible.

\begin{figure*}[t]
    \centering
    \includegraphics[width =.9\textwidth]{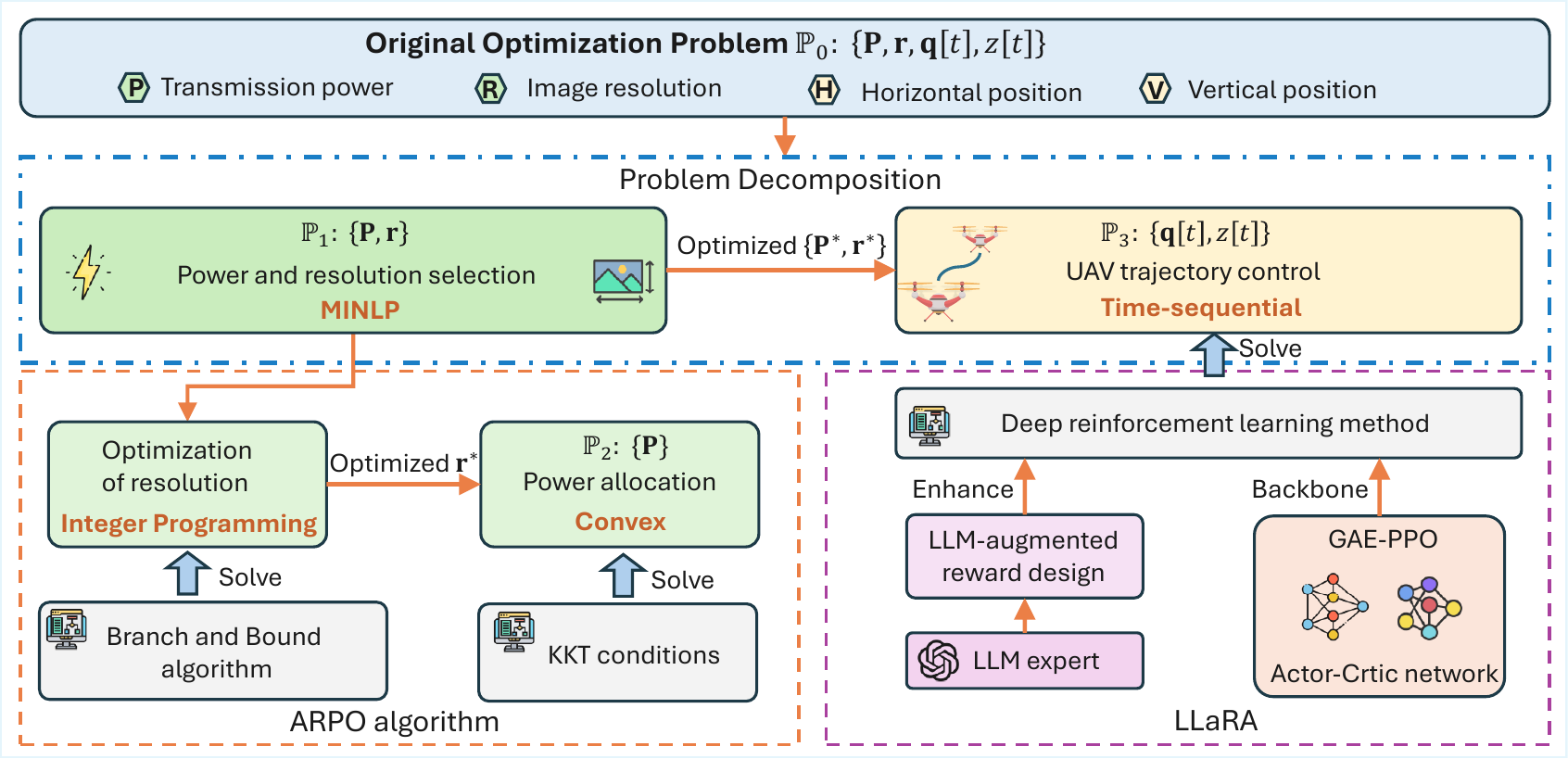}
    \caption{The framework of our proposed hierarchical ARPO-LLaRA optimization framework. At the start of uplink session, ARPO determine image resolutions $\mathbf{r}$ and powers $\mathbf{P}$ for transmission using the B\&B algorithm and KKT conditions, respectively. Then, LLaRA uses an LLM-assisted DRL method for planning the slot-level UAV trajectory.}
    \label{fig:optimization}
\end{figure*}

To tackle with the challenges of $\mathbb{P}_0$, we adopt a two-step hierarchical framework. First, we optimize the transmit power and image resolution via the ARPO algorithm. Second, we optimize the UAV trajectory using the LLaRA. This decomposition reduces overall complexity and allows each subproblem to be solved with the most suitable method. Fig.~\ref{fig:optimization} illustrates the overall ARPO-LLaRA optimization framework.

\section{The proposed ARPO Algorithm for resolution and power control}\label{sec:ARPO}

This section introduces the ARPO algorithm, which jointly optimizes image resolutions $\mathbf{r}:=\{r_n\}_{n\in\mathcal{N}}$ and transmit powers $\mathbf{P}:=\{P_n\}_{n\in\mathcal{N}}$. 
Specifically, given the initial UAV position $\{\bar{\mathbf{q}}, \bar{z}\}$ at the start of the uplink session, ARPO chooses \emph{slot-invariant} $(\mathbf{r},\mathbf{P})$. This aligns with our Assumption~\ref{assumption:1} that fixes image resolution and power per query/user within an uplink session, and cleanly separates the slower per-session controls $(\mathbf{r},\mathbf{P})$ from the slot-level trajectory evolution. 

Based on the above statements, we now simplify the original problem $\mathbb{P}_0$ to facilitate optimization. When choosing $(\mathbf{r},\mathbf{P})$, we evaluate the channel at the current pose $\{\bar{\mathbf{q}}, \bar{z}\}$ and treat it as constant over the decision interval. Under this quasi-static approximation, the per-slot rate becomes $R_n[t]\equiv R_n$, and the uplink time simplifies to the closed form:
\begin{align}
T_n^{\text{up}}(r_n,P_n)
= \frac{D_n(r_n)}{B_n\log_2\!\big(1+\tfrac{P_n\,|h_n|^2}{\sigma_n^{2}}\big)},
\label{eq:arpo_up}
\end{align}
where $D_n(r_n)$ is the payload at resolution $r_n$ and $B_n$ is the allocated bandwidth fixed by the system operator. After that, the total latency for user $n$ can be rewritten as:
\begin{align}
    T^{\text{total}}_n(r_n,P_n) = 
    T_n^{\text{up}}(r_n,P_n)
    + T_n^{\text{proc}}(r_n) 
    + T_n^{\text{down}},\label{equa:T_r_P}
\end{align}
Following~\cite{li2025resource}, we introduce an auxiliary variable $\tau$ to replace the term $\max_{n \in \mathcal{N}}~T^{\text{total}}_n$ in \eqref{Original_Problem} with a new constraint \eqref{constraint:T}. 
Problem $\mathbb{P}_0$ is transformed into an equivalent form $\mathbb{P}_1$:
\begin{subequations}\label{Power_Resolution_Optimization}
\begin{align}
\mathbb{P}_1:\quad 
\min_{\mathbf{r},\mathbf{p},\tau}~&\Big\{ \tau + \zeta \sum_{n \in \mathcal{N}} P_n \Big\}\tag{\ref{Power_Resolution_Optimization}} \\ 
\mathrm{s.t.} \quad 
& \text{(\ref{constraint:accuracy})},~\text{(\ref{constraint:power})},~\text{(\ref{constraint:resolution})},~
\nonumber\\
& T^{\text{total}}_n(r_n,P_n) \le \tau, \quad \forall n.\label{constraint:T}
\end{align}
\end{subequations}
However, problem $\mathbb{P}_1$ is still an MINLP that cannot be solved directly, and $T^{\text{total}}_n(r_n,P_n)$ in \eqref{constraint:T} involves a non-convex term.

To facilitate the solution, we decouple variables $\{\mathbf{r},\mathbf{P}\}$ and optimize them in an alternating manner. First, recalling that the inference speed function $v(r_n)$ is non-decreasing with resolution $r_n$, we can derive that the objective function \eqref{Power_Resolution_Optimization} is also non-decreasing with $r_n$. Hence, to reduce overall latency, we choose the smallest resolutions that satisfy the accuracy constraint \eqref{constraint:accuracy}. 
In practice, VLMs usually support only a small finite set of resolutions (e.g., $4$-$5$ choices such as $\{336^2,448^2,1024^2,1536^2\}$ for LLaVA-HR~\cite{luo2024feast}).
Hence, the optimization on $\mathbf{r}$ is lightweight and can be implemented via a Branch and Bound (B\&B) algorithm~\cite{zoppei2022branch} effectively.

Next, we can substitute the obtained $\mathbf{r}^*$ from the B\&B algorithm into problem $\mathbb{P}_1$ and reformulate it as follows:
\begin{subequations}\label{Problem1_1}
\begin{align}
\mathbb{P}_2:\quad 
\min_{\mathbf{p},\tau} ~ &\Big\{ \tau + \zeta \sum_{n \in \mathcal{N}} P_n \Big\}\tag{\ref{Problem1_1}} \\ 
\mathrm{s.t.} \quad 
& \text{(\ref{constraint:power})},
\nonumber\\
& T^{\text{total}}_n(r_n^*,P_n) \le \tau, \quad \forall n.\label{constraint:T2}
\end{align}
\end{subequations}
With both the objective and constraints of $\mathbb{P}_2$ being convex, it now becomes a convex problem. Consequently, KKT conditions can be applied directly to derive the optimal solution. With $\bm{\iota} := [\iota_n|_{n \in \mathcal{N}}]$ and $\bm{\omega} := [\omega_n|_{n \in \mathcal{N}}]$ denoting the Lagrange multipliers, the corresponding Lagrangian function is given by:
\begin{align}
&\mathcal{L}_1(\bm{P},\tau,\bm{\lambda},\mathbf{\iota},\mathbf{\omega}) = \tau + \zeta \sum_{n \in \mathcal{N}} P_n  + \sum_{n \in \mathcal{N}}  \iota_n \cdot (P_n-P_n^{\max}) \nonumber\\
&+ \sum_{n \in \mathcal{N}} \omega_n \cdot (T^{\text{total}}_n(r_n^*,P_n) - \tau).
\end{align}
The KKT conditions are as follows:
\begin{subequations}
\begin{numcases}{}
    \frac{\partial \mathcal{L}_1}{\partial P_n} = \zeta + \iota_n 
    -\omega_n g_n(P_n,r_n^*) = 0, \quad \forall n, \label{Lagrange:partial_P} \\
    \frac{\partial \mathcal{L}_1}{\partial \tau} = 1-\sum_{n \in \mathcal{N}} \omega_n = 0, \label{Lagrange:partial_tau}\\
    \iota_n \cdot (P_n-P_n^{\max})=0, \quad \forall n, \label{Lagrange:iota} \\
    \omega_n \cdot (T^{\text{total}}_n(r_n^*,P_n) - \tau) =0, \quad \forall n, \label{Lagrange:omega}
\end{numcases}
\end{subequations}
where $\mathcal{L}_1(\bm{P},\tau,\bm{\iota},\bm{\omega})$ is abbreviated as $\mathcal{L}_1$, and $g_n(P_n,r_n^*)$ denotes the derivative of $T^{\text{total}}_n(r_n^*,P_n)$ with respect to $P_n$:
\begin{align}
    g_n(P_n,r_n^*) = 
\frac{D(r_n^*) h_n}{B_n \sigma^2 (1 + \frac{h_n P_n}{\sigma^2}) 
\ln 2  [\log_2(1 + \frac{h_n P_n}{\sigma^2})]^2}.
\end{align}
With conditions \eqref{Lagrange:partial_P}-\eqref{Lagrange:omega}, \textbf{Proposition \ref{proposition:KKT}} is given to find the optimal solution to problem $\mathbb{P}_2$.


\begin{proposition}
The optimal solution $\mathbf{P}^*=\{P_n^*\}_{n\in\mathcal{N}}$ and $\tau^*$ to problem $\mathbb{P}_2$ are expressed as:
\begin{align}
    &P_n^* =
\begin{cases}
    P_n^{\max}, & \text{if } \zeta=0, \\
    \min\{P_n(\tau^*),\, P_n^{\max}\}, & \text{if } \zeta>0,
\end{cases}
\label{equa:optimal_p}\\
    &\tau^* \text{ is the solution to: } \sum_{n \in \mathcal{N}} \frac{\zeta}{g_n(P_n(\tau),r_n^*)} =1,
\end{align}
where
\begin{subequations}
\begin{numcases}{}
    P_n(\tau) = \frac{\sigma^2}{h_n} \big( 2^{\frac{D(r_n^*)}{B_n (\tau - \Gamma_n(r_n^*) )}} - 1 \big),\\
    \Gamma_n(r_n^*) = T_n^{\text{proc}}(r_n^*) + T_n^{\text{down}}.
\end{numcases}
\end{subequations}
Specifically, the solution $\tau^*$ can be efficiently obtained using a 1-D bisection search, and the corresponding $\hat{P}_n(\hat{\tau})$ can be computed via \eqref{equa:optimal_p}. \textit{Thus, KKT removes the need for a multi-dimensional search, i.e., we evaluate $P_n(\tau)$ in closed form and only line-search over $\tau$.}
\label{proposition:KKT}
\end{proposition}

\begin{proof} 
Recall that $\zeta$ in \eqref{Lagrange:partial_P} is a pre-determined coefficient for latency-power trade-off. If $\zeta = 0$, the objective ignores the power cost. According to the condition~\eqref{Lagrange:partial_P}, the optimal power is achieved at the upper bound $P_n^*=P_n^{\max}$.
Otherwise, from condition~\eqref{Lagrange:partial_P}, we obtain:
\begin{align}
    \zeta + \iota_n - \omega_n g_n(P_n, r_n^*) = 0, \quad \forall n,
\end{align}
where only $\omega_n >0$ makes the equation hold. Substitute $\omega_n>0$ into \eqref{Lagrange:omega}, we can obtain: $T_n^{\text{total}}(P_n, r_n^*) = \tau$, which yields an implicit relation between $P_n$ and $\tau$. Solving for $P_n$ gives:
\begin{align}
    P_n(\tau) = \frac{\sigma^2}{h_n} \left(2^{\frac{D(r_n^*)}{B_n(\tau - \Gamma_n(r_n^*))}} - 1\right),
\end{align}
where $\Gamma_n(r_n^*) = T_n^{\mathrm{proc}}(r_n^*) + T_n^{\mathrm{down}}$. Substituting $P_n(\tau)$ into the stationarity condition with $\iota_n = 0$ leads to:
\begin{align}
    \omega_n(\tau) = \frac{\zeta}{g_n(P_n(\tau), r_n^*)}.
\end{align}
Then, from \eqref{Lagrange:partial_tau}, we obtain the condition:
\begin{align}
    \sum_{n \in \mathcal{N}} \omega_n(\tau) = 1,
\end{align}
which can be efficiently solved for $\tau^*$ due to the monotonicity of the left-hand side. The corresponding power allocation is then given by $P_n^* = \min\{P_n(\tau^*),\, P_n^{\max}\}$ to satisfy the constraint $P_n \leq P_n^{\max}$, and the corresponding dual variables can be obtained from the KKT conditions.
\end{proof}




\begin{algorithm}
\caption{ARPO algorithm for problem $\mathbb{P}_1$} 
\label{Algorithm:ARPO}
\KwIn{
MINLP problem $\mathbb{P}_1$;
}
\KwOut{Optimal solution $\{\mathbf{r}^*,\mathbf{P}^*,\tau^*\}$;}

Search the minimum optimal $\mathbf{r}^*=\{r_n^*\}_{n \in \mathcal{N}}$ for each user $n$ using the B\&B algorithm, s.t., $A(r_n^*)\geq A_n^{\min}$;

Given the obtained $\mathbf{r}^*$, transform $\mathbb{P}_1$ into a convex problem $\mathbb{P}_2$;

Apply KKT conditions on $\mathbb{P}_2$ to obtain \eqref{Lagrange:partial_P}-\eqref{Lagrange:omega};

Derive the optimal solution $\{\mathbf{P}^*,\tau^*\}$ to $\mathbb{P}_2$ according to \textbf{Proposition \ref{proposition:KKT}} when coefficient $\zeta$ is given;

\Return{$\{\mathbf{r}^*,\mathbf{P}^*,\tau^*\}$}
\end{algorithm}

\section{The Proposed LLaRA Method for trajectory optimization}\label{sec:LLaRA}

This section details our proposed LLaRA method for optimizing the UAV trajectory. With the obtained $\{\mathbf{r}^*,\mathbf{p}^*\}$ from the ARPO algorithm, now we focus on the optimization of the UAV trajectory $\{(x[t],y[t],z[t])\}_{t \in \mathcal{T}}$. Accordingly, the original problem $\mathbb{P}_0$ can be simplified as follows:
\begin{subequations}\label{Trajectory_Optimization}
\begin{align}
\mathbb{P}_3:
\min_{\{x[t],y[t],z[t]\}}~&\left\{\max_{n \in \mathcal{N}}~
    T_n^{\text{total}}\big(r_n^*,P_n^*\big)\right\}
    \tag{\ref{Trajectory_Optimization}} \\ 
\mathrm{s.t.} \quad 
& \text{(\ref{constraint:altitude})},~\text{(\ref{constraint:speed_xy})},~\text{(\ref{constraint:speed_z})}. \nonumber
\end{align}
\end{subequations}
To solve problem $\mathbb{P}_3$, we propose LLaRA, which employs the LLM as a reward designer for enhancing the traditional RL ability. We present the design of LLaRA in the following subsections.


\begin{figure*}
    \centering
    \includegraphics[width =.9\textwidth]{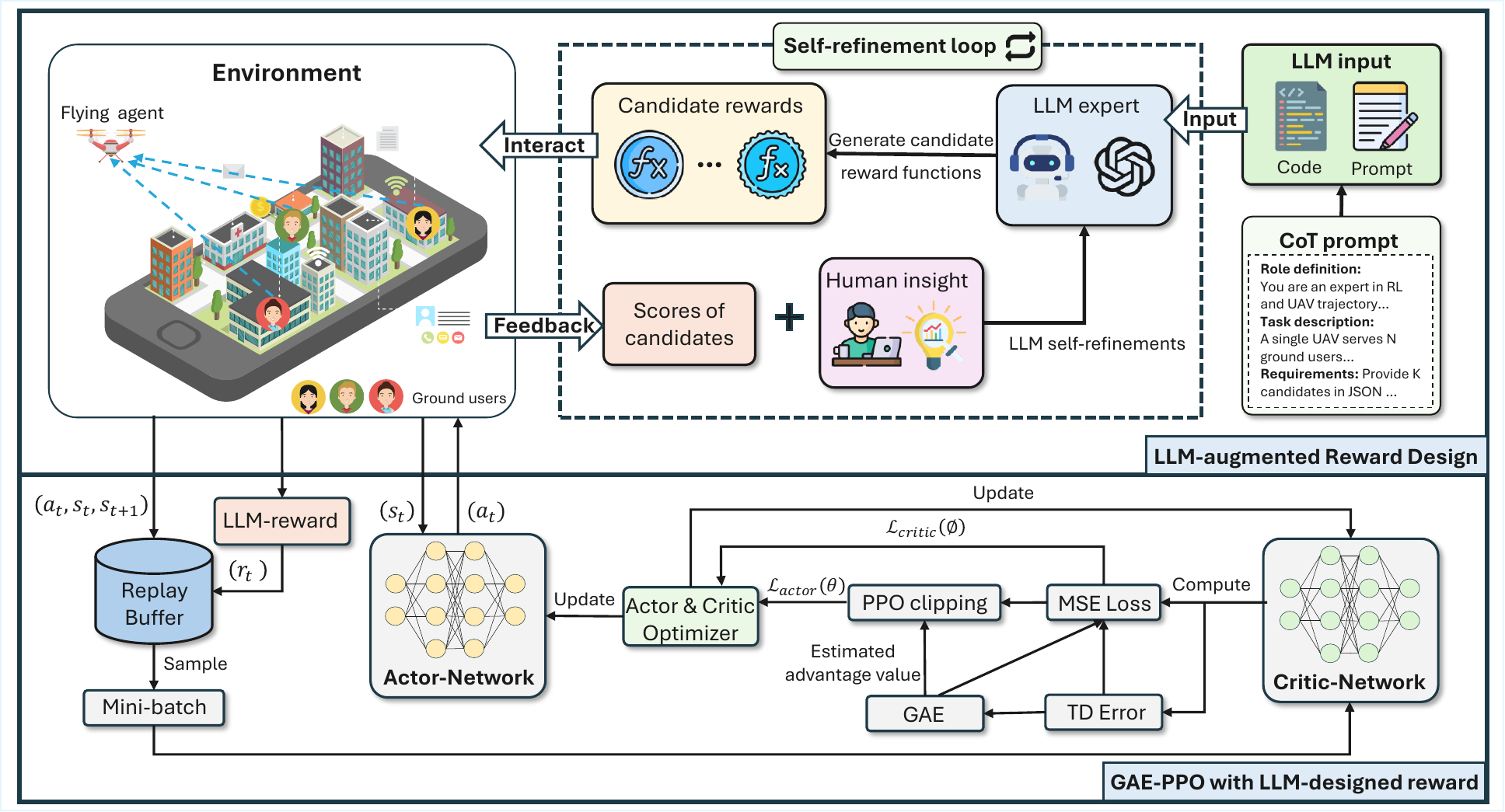}
    \caption{The workflow of LLaRA. The LLM-augmented reward design employs an LLM expert to generate and iteratively refine the candidate reward functions. The GAE-PPO strategy updates the Actor and Critic networks using the feedback provided by the refined LLM-designed reward.}
    \label{fig:LLM-PPO}
\end{figure*}

\subsection{An overview of LLaRA}

Generally speaking, LLaRA integrates traditional DRL with LLM-augmented reward design to optimize UAV trajectories under dynamic LAENet conditions. 
Traditional DRL often suffers from heuristic reward designs, which can hinder convergence and limit generalization. LLaRA addresses this limitation by combining a stable DRL backbone with the reasoning and code generation capabilities of advanced LLMs. The overall workflow is illustrated in Fig.~\ref{fig:LLM-PPO}.



Specifically, LLaRA proceeds in two separate stages. First, before deployment, we apply an LLM in an \textit{offline} reward design loop, as shown in the upper part of Fig.~\ref{fig:LLM-PPO}.
The employed LLM receives structured prompts containing the system information, and outputs candidate reward functions. These candidate rewards are then evaluated in the training environment, scored according to their effectiveness, and iteratively refined through an LLM self-improvement process~\cite{kwon2023reward,ma2023eureka}. By incorporating both performance feedback and optional human guidance, the LLM progressively improves the reward design toward better alignment with the optimization goals.

Second, once a refined reward function design is selected, it is incorporated into the Proximal Policy Optimization (PPO) backbone~\cite{schulman2017proximal} for trajectory learning. PPO alternates between data collection and policy update, with the clipped surrogate objective ensuring training stability. 
The advantage estimates used in PPO are computed via Generalized Advantage Estimation (GAE), which further improves sample efficiency and variance reduction. Importantly, because the LLM-assisted reward generation occurs offline, no additional inference latency is introduced during online UAV operation. During deployment, trajectory decisions are obtained by a simple forward pass of the trained actor network, ensuring that the method remains efficient in real-time settings~\cite{11049053}.

\subsection{LLM-augmented Reward Design.}
We present a novel LLM-augmented reward design scheme,
where the LLM acts as an automated reward design expert. Our approach employs a Chain-of-Thought (CoT)-enhanced prompt engineering strategy~\cite{cai2025large} to guide the LLM in understanding our system model, identifying optimization factors, and generating executable codes for candidate reward functions. These candidate functions are then iteratively evaluated, refined, and integrated into our GAE-PPO method. In the following, we provide a detailed description of the whole process, covering prompt construction, reward generation, evaluation, and iterative refinement. Fig.~\ref{fig:prompt_design} presents an instance of LLM prompts used for reward generation and subsequent refinement.

\subsubsection{Prompt Engineering for Reward Design}
To guide the LLM in effectively designing reward functions, the first step is to ensure that the LLM understands its role properly. Hence, we adopt a CoT-enhanced prompt structure consisting of the following elements:

\textbf{Role definition.} Firstly, the LLM is specified as a professional reward designer with expertise in DRL and LAENets. Its tasks include: (i) understanding the underlying system model, (ii) reasoning over observations and agent–environment interactions, and (iii) generating executable Python code for reward functions. Normative constraints are also imposed, such as prohibiting assumptions based on ungiven information and prioritizing the most relevant factors in the reward to mitigate potential LLM hallucinations~\cite{ji2023survey}. To facilitate seamless integration with our PPO pipeline, the response format is also required to be standardized (e.g., in JSON), enabling automated validation and execution.

\textbf{Task description.} Task description is to help LLM grasp a comprehensive understand on the studied problem, including background, system model, and objective function, etc. To bridge the gap between natural language and programmatic understanding, we further embed code snippets (e.g., MDP design) alongside textual descriptions as a unified input to LLM. This hybrid expression motivates LLM to focus on task-relevant contextual information, while reducing the likelihood of producing overly generic reward functions.

\subsubsection{Initial Reward Function Generation} With our crafted prompts (i.e., role definition and task description), the LLM generates a set of candidate reward functions $\mathcal{R}_{\mathrm{LLM}}$:
\begin{align}
    \mathcal{R}_{\mathrm{LLM}}=F_{\mathrm{Gen}}(\mathcal{P}, \mathcal{C}; \Theta),\label{equa:LLM_generation}
\end{align}
where $F_{\mathrm{Gen}}(\cdot)$ is the function for generating reward functions from the prompt, $\mathcal{P}$ denotes the textual part in role definition and task description, $\mathcal{C}$ denotes the embedded code snippets, and $\Theta$ is the LLM parameter set. Specifically, each candidate reward $R_i \in \mathcal{R}_{\mathrm{LLM}}$ should satisfy the predefined return type and align with the optimization problems (e.g., incorporating both objectives and penalties for constraint violations).

The obtained candidate reward functions $\mathcal{R}_{\mathrm{LLM}}$ can be evaluated through interactions with the devised environment. In detail, their effectiveness are quantified by a score set $\mathcal{S}_{\mathrm{LLM}}$:
\begin{align}
    \mathcal{S}_{\mathrm{LLM}} = F_{\mathrm{Eval}}(\mathcal{R}_{\mathrm{LLM}}; \mathcal{E}, \Phi),\label{equa:evaluation}
\end{align}
where $F_{\mathrm{Eval}}(\cdot)$ is the function for evaluation, $\mathcal{E}$ denotes the DRL environment, and $\Phi$ represents the evaluation configuration (e.g., training episodes and performance metrics). 
Each score $S_i \in \mathcal{S}_{\mathrm{LLM}}$ refers to the performance of candidate $R_i \in \mathcal{R}_{\mathrm{LLM}}$, and helps to guide refinements to reward functions.

\subsubsection{Iterative Self‑Refinement} To further improve the quality of generated reward functions, we introduce a self‑refinement strategy that allows the LLM to iteratively make improvements to reward functions $\mathcal{R}_{\mathrm{LLM}}$. Specifically, the self-refinement prompt fed to the LLM will include: (i) evaluation results of all candidate functions $\mathcal{S}_{\mathrm{LLM}}$, and (ii) human insights $\mathcal{H}$ (optional) to guide the refinements. The refinement process is as follows:
\begin{equation}
    \mathcal{R}_{\mathrm{LLM}}' = F_{\mathrm{Ref}}(\mathcal{S}_{\mathrm{LLM}}, \mathcal{H}; \mathcal{R}_{\mathrm{LLM}}, \Theta),\label{equa:refinements}
\end{equation}
where $\mathcal{R}_{\mathrm{LLM}}'$ denotes the set of refined candidates, and $F_{\mathrm{Ref}}(\cdot)$ refers to the function for refinements. In each iteration, the LLM is re‑prompted with the evaluation results $\mathcal{S}_{\mathrm{LLM}}$, focusing on refining the top‑performing candidates. Human preferences are optional but particularly valuable when multiple candidates in $\mathcal{R}_{\mathrm{LLM}}$ exhibit comparable performance, as insights provide additional guidance for steering the LLM toward more effective reward designs.

After going through a certain number of refinement rounds as shown in \eqref{equa:refinements}, we can perform a final evaluation and select the best‑performing reward function $R_{\mathrm{LLM}}^*$:
\begin{equation}
    R_{\mathrm{LLM}}^* = \arg\max_{R_i \in \mathcal{R}_{\mathrm{LLM}}'} S_i,\label{equa:best_score}
\end{equation}
and the selected $R_{\mathrm{LLM}}^*$ is adopted as the final reward function of LLaRA. We summarize the overall process in Algorithm~\ref{Algorithm:LLaRA}.

\begin{figure}[t]
    \centering
    \includegraphics[width =\linewidth]{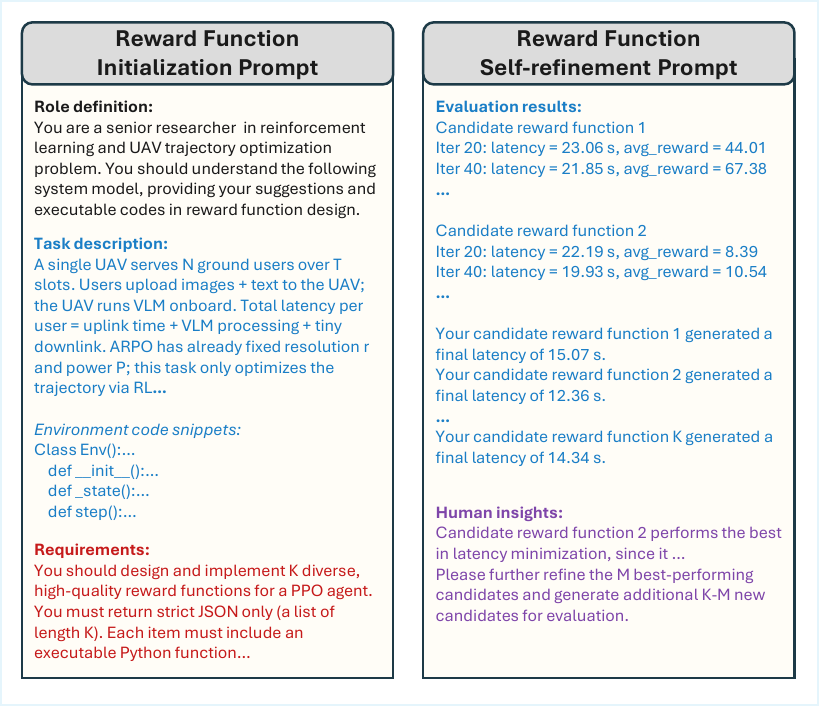}
    \caption{Instance prompts used in the initialization and evolution of our LLM-assisted reward design.}
    \label{fig:prompt_design}
\end{figure}

\begin{algorithm}
\caption{LLM-augmented Reward Design} 
\label{Algorithm:LLaRA}
\KwIn{
    Role definition $\mathcal{P}_{\mathrm{role}}$, 
    Task description $\mathcal{P}_{\mathrm{task}}$, 
    Requirements $\mathcal{P}_{\mathrm{req}}$,
    Code snippets $\mathcal{C}$,
    \newline
    DRL environment $\mathcal{E}$, 
    Evaluation policy $\Phi$,
    \newline
    Human insights $\mathcal{H}$, 
    LLM parameters $\Theta$;
}
\KwOut{Final reward function $R_{\mathrm{LLM}}^*$;}

Construct the CoT-enhanced prompt $\mathcal{P}$: $\mathcal{P} \leftarrow (\mathcal{P}_{\mathrm{role}}, \mathcal{P}_{\mathrm{task}},\mathcal{P}_{\mathrm{req}})$;

Generate a set of candidate reward functions $\mathcal{R}_{\mathrm{LLM}}$ via~\eqref{equa:LLM_generation} by feeding $\mathcal{P}$ and $\mathcal{C}$ into the LLM;

Evaluate $\mathcal{R}_{\mathrm{LLM}}$ by interactions with environment $\mathcal{E}$ and policy $\Phi$, obtain a score set $\mathcal{S}_{\mathrm{LLM}}$ via \eqref{equa:evaluation};

\While{termination criterion not met}{
    Given $\mathcal{S}_{\mathrm{LLM}}$ and human insights $\mathcal{H}$, LLM refines the candidates function into $\mathcal{R}_{\mathrm{LLM}}'$ via \eqref{equa:refinements};

    Evaluate refined $\mathcal{R}_{\mathrm{LLM}}'$ via \eqref{equa:evaluation} to obtain a new score set $\mathcal{S}_{\mathrm{LLM}}$;

    Update candidates $\mathcal{R}_{\mathrm{LLM}}$ with new $\mathcal{R}_{\mathrm{LLM}}'$;
}
Select the best-performing $\mathcal{R}^*_{\mathrm{LLM}}$ via \eqref{equa:best_score};

\Return{$R_{\mathrm{LLM}}^*$}

\end{algorithm}

\subsection{MDP Formulation with LLM-designed Rewards}

We formulate problem $\mathbb{P}_3$ as an MDP to facilitate the RL solution. Specifically, the MDP is often defined by a tuple $\langle \mathcal{S}, \mathcal{A}, \mathcal{R}, \mathcal{M}, \rho \rangle$, where $\mathcal{S}$ denotes for the state space, $\mathcal{A}$ the action space, $\mathcal{R}$ the reward design, $\mathcal{M}$ the state transition model, $\rho\in(0,1)$ the discount factor. Among these elements, $\mathcal{S}$, $\mathcal{A}$, and $\mathcal{R}$ are pivotal in the agent’s learning process, and their definitions are detailed below.

\subsubsection{State Space $\mathcal{S}$}\label{subsec:state}
We design the state space $\mathcal{S}$ to contain the key environmental factors for agent decision-making. 
To better describe spatial relationships, we adopt relative positions to represent all positional relations in our system. Accordingly, the state $\mathbf{s}_t$ of the UAV at time slot $t$ is defined as follows:
\begin{align}
\mathbf{s}_t =
\Bigl[
&
\underbrace{\{(x[t]\!-\!x_n,y[t]\!-\!y_n,z[t]\!-\!h_n)\}_{n\in\mathcal{N}}}_{\text{UAV-user relative positions}};~\{r_n\}_{n\in \mathcal{N}};\nonumber\\
&\hspace{10pt}
\{P_n\}_{n\in \mathcal{N}};~\{h_n[t]\}_{n\in \mathcal{N}};~\{d_n[t]\}_{n\in \mathcal{N}}
\Bigr], \quad \forall t.\label{equa:state}
\end{align}
$\{r_n\}_{n\in \mathcal{N}}$ denotes the image resolutions; $\{P_n\}_{n\in \mathcal{N}}$ is the transmit power; $\{h_n[t]\}_{n\in \mathcal{N}}$ is the channel gain at slot $t$; $\{d_n[t]\}_{n\in \mathcal{N}}$ is the remaining data sizes waiting for transmission.

\subsubsection{Action Space $\mathcal{A}$}\label{subsec:action}
The action space $\mathcal{A}$ is directly related to optimization variables. After obtaining the state information $\mathcal{S}$, the UAV needs to plan its next move based on the policy distribution. The action $\mathbf{a}_t$ at time slot $t$ consists of:
\begin{equation}
\mathbf{a}[t] = \left\{(\Delta x[t], \Delta y[t], \Delta z[t])\right\},\quad \forall t,
\end{equation}
where $(\Delta x[t], \Delta y[t], \Delta z[t])$ denotes a movement vector for the UAV mobility. It is bounded by predefined values to comply with related constraints \eqref{constraint:altitude}-\eqref{constraint:speed_z} as follows:
\begin{subequations}
\begin{numcases}{}
h^{\min} \leq z[t] \leq h^{\max},\quad \forall t,\\
\sqrt{(\Delta x[t])^2+(\Delta y[t])^2} \leq \alpha \cdot V_{xy}^{\max},\quad \forall t, \\
|\Delta z[t]| \leq \alpha \cdot V_z^{\max},\quad \forall t.
\end{numcases}
\end{subequations}

\subsubsection{LLM-designed Reward $\mathcal{R}_{\mathrm{LLM}}$}\label{subsec:reward}
Reward design is central to an MDP, as it evaluates state–action pairs and guides the agent toward the optimal policy. In this work, we utilize a ``risk-aware'' reward function $R_{\mathrm{LLM}}^{\mathrm{risk}}$ refined by the LLM
and detail its expression as follows:
\begin{align}
    R_{\mathrm{LLM}}^{\mathrm{risk}}=
    &-\mathrm{VaR}_{q}\big(\mathbf{d}[t]\big)+\mu \sum_{n \in \mathcal{N}} \min\{ d_n[t],\, \alpha R_n[t] \} \nonumber\\
    &+ \gamma_{d} \cdot \Delta\mathrm{dist}[t],\label{equa:LLM_reward}
\end{align}
where $\mathrm{VaR}_{q}\big(\mathbf{d}[t]\big)=\inf\!\left\{\tau:\; \frac{1}{|\mathcal{N}|}\sum_{n\in\mathcal{N}}\mathbf{1}\{d_n[t]\le \tau\}\ge q\right\}$ is the empirical $q$-quantile of user backlogs, penalizing the tail and aligning with the worst-case latency in \eqref{Trajectory_Optimization}; the second term $\sum_{n\in\mathcal{N}}\min\{d_n[t],\,\alpha R_n[t]\}$ counts only transmitted data from unfinished users to preserve overall system efficiency; and $\Delta\mathrm{dist}[t]=\|(\mathbf{q}[t],z[t])-(\mathbf{w}_{n_t},h_{n_t})\|_2-\|(\mathbf{q}[t{+}1],z[t{+}1])-(\mathbf{w}_{n_t},h_{n_t})\|_2$ with $n_t=\arg\max_{n} d_n[t]$ rewards motion toward the bottleneck user. Here $q\in(0,1)$ and $\mu,\gamma_d\ge 0$ are weights.

Notably, $R_{\mathrm{LLM}}^{\mathrm{risk}}$ designed by LLM includes not only a tail-aware term $\mathrm{VaR}_{q}\big(\mathbf{d}[t]\big)$ motivated by \eqref{Trajectory_Optimization}, but also incorporates potentially influential components on overall effective throughput and the distance change between the bottleneck user and UAV. This formulation can help the agent to observe its action values from a more comprehensive view, and hence increase the convergence performance of the learned policy.

\section{Complexity Analysis}\label{sec:complexity}
Here we analyze the computational complexity of ARPO and LLaRA in our optimization framework, and discuss the overall complexity in the real-time decision-making process.


\subsection{Complexity of ARPO}
The ARPO algorithm optimizes resolution selection $\mathbf{r}$ and power allocation $\mathbf{P}$ in an alternating fashion. Firstly, the optimal $\mathbf{r}^*$ is determined using a B\&B method over the discrete set $\mathcal{R}^{\mathrm{res}}$. Let $N$ denote the number of users and $J$ the number of candidate resolutions. In the worst case, the computational complexity of this search is $\mathcal{O}\big(J^N\big)$~\cite{morrison2016branch}, but the B\&B pruning mechanism often reduces the effective search space. If $L$ denotes the estimated number of explored branches, the practical time complexity is approximately $\mathcal{O}\big(L\cdot N\big)$. Given $\mathbf{r}^*$, the power allocation subproblem is convex and solved via a one-dimensional bisection search over $\tau$, yielding a complexity of $\mathcal{O}\big(N\cdot \log(1/\epsilon)\big)$, where $\epsilon$ is the convergence tolerance. Finally, an additional $\mathcal{O}\big(N\big)$ complexity accounts for the final back-substitution of $\tau^*$ to compute $\mathbf{P}^*$. Hence, the overall complexity of ARPO is $\mathcal{O}\big((\log(1/\epsilon)+L+1)\cdot N\big)$.

\subsection{Complexity of LLaRA}
The complexity of the LLaRA algorithm consists of two main components: the complexity of the GAE-PPO backbone, and the complexity of LLM-augmented reward design.


\subsubsection{GAE-PPO Complexity}
The complexity of GAE-PPO is primarily determined by the architecture of the deep neural networks (DNNs) used in the actor and critic networks. The complexity of DNNs is expressed as $\mathcal{O}\big(\sum_{p=1}^{P} n_{p-1} n_p \big)$, where $n_p$ denotes the number of neurons in the $p$-th layer and $P$ is the total number of layers in the actor–critic networks. In addition, GAE introduces an additional complexity of $\mathcal{O}\big(M\big)$ per iteration, where $M$ denotes the number of time steps used in advantage computation~\cite{11049053}. Hence, the overall complexity of GAE-PPO is $\mathcal{O}\big(M \cdot \sum_{p=1}^{P} n_{p-1} n_p \big)$.


\subsubsection{LLM-Enhanced Reward Design Complexity}
The reward design involves prompting the LLM and evaluating candidate reward functions. Let $K$ denote the number of reward candidates, $I$ denote the number of refinement rounds, and $C_{\mathrm{LLM}}$ denote the average cost of one LLM call. Then, the overall complexity is $\mathcal{O}\big(I \cdot K \cdot (C_{\mathrm{LLM}} + C_{\mathrm{eval}})\big)$, where $C_{\mathrm{eval}}$ is the cost of evaluating one candidate in the DRL environment.
Since both LLM calls and evaluations are performed before deployment, its impact on real-time deployment is negligible.

\textbf{Overall Complexity in Real-time Decision-Making.} 
The major computational overhead of the proposed ARPO-LLaRA framework arises from the training of the PPO backbone and the LLM-assisted reward refinement. However, these phases are performed offline and prior to deployment. During real-time operation, the decision-making process only involves ARPO and a single forward pass of LLaRA, bringing an overall time complexity of $\mathcal{O}\big((\log(1/\epsilon)+L+1)\cdot N+M \cdot \sum_{p=1}^{P} n_{p-1} n_p \big)$.

\section{Experiments}\label{sec:experiments}
We conduct extensive experiments to validate the proposed framework. Section~\ref{subsec:settings} details experimental parameter settings; 
Section~\ref{subsec:accuracy} quantifies the impact of input resolution on VLM accuracy and efficiency; Sections~\ref{subsec:comparative_analysis} to \ref{subsec:power_bandwidth} present different simulation results.


\subsection{Experimental Settings}\label{subsec:settings}

\subsubsection{VLM Settings} To investigate how input image resolution affects the performance of VLMs, we conduct an empirical study on two representative architectures: LLaVA-1.5~\cite{liu2024improved} and its high-resolution variant LLaVA-HR~\cite{luo2024feast}. Both models are evaluated on the TextVQA benchmark~\cite{singh2019towards}, which is a widely used dataset designed to evaluate multimodal inference and aligns well with our considered LAENet users' tasks. Interested  readers can refer to \url{https://github.com/luogen1996/LLaVA-HR} for more technical details of LLaVA models.

\subsubsection{Reward-Design Settings}\label{subsec:reward_settings}
Reward design is performed offline with the help of an LLM: We prompt GPT-4o~\cite{hurst2024gpt} to synthesize and refine executable reward candidates (code/JSON). Specifically, we use GPT-4o because prior work has verified that it reliably produces executable reward code~\cite{xie2025erfsl} and can serve as an automatic monitor in RL reward pipelines~\cite{baker2025monitoring}.

\subsubsection{Parameter Settings}
Here we give the default parameter settings, primarily adapted from~\cite{zhan2020joint}. Specifically, we consider a UAV-enabled LAENet with $N=4$ ground users distributed within an $1000\times1000~\mathrm{m}^2$ square area, facilitating clear visualization in square-shaped figures. The UAV's initial location is set as $(-500,-500,150)$. The minimum and maximum allowable altitudes for the UAV, i.e., $h^{\min}$ and $h^{\max}$, are set as $50~\mathrm{m}$ and $300~\mathrm{m}$, respectively. Ground users are at altitudes $h_n=0~\mathrm{m}$. The elemental time slot length $\alpha$ is $1~\mathrm{s}$ and the number of slots $T$ is $50$. The UAV's maximum allowable horizontal and vertical speeds, i.e., $V_{xy}^{\max}$ and $V_{z}^{\max}$, are $100~\mathrm{m/s}$ and $20~\mathrm{m/s}$, respectively.
For the communication model parameters, we adopt a standard large-scale path-loss model with exponents $\gamma_{\mathrm{LoS}}$ and $\gamma_{\mathrm{NLoS}}$ as $2$, parameters $(a,b)=(4.88,\,0.43)$, reference channel gain $\beta_0 = -50~\mathrm{dB}$, and noise power $\sigma^2 = -100~\mathrm{dBm}$. Each user is allocated a bandwidth of $B_n = 1~\mathrm{MHz}$ and a maximum transmit power of $P_n^{\max} = 0.1~\mathrm{W}$. The downlink time $T_n^{\text{down}}$ is set as $0.1$ s.

\subsection{Empirical Study on Resolution-Aware Model Performance}\label{subsec:accuracy}

To better illustrate the impact of visual input resolution on VLM performance, we conduct both qualitative and quantitative analyses. 
Fig.~\ref{fig:LLaVA} summarizes our empirical study.

\begin{figure*}[t]
    \centering
    \includegraphics[width =.9\linewidth]{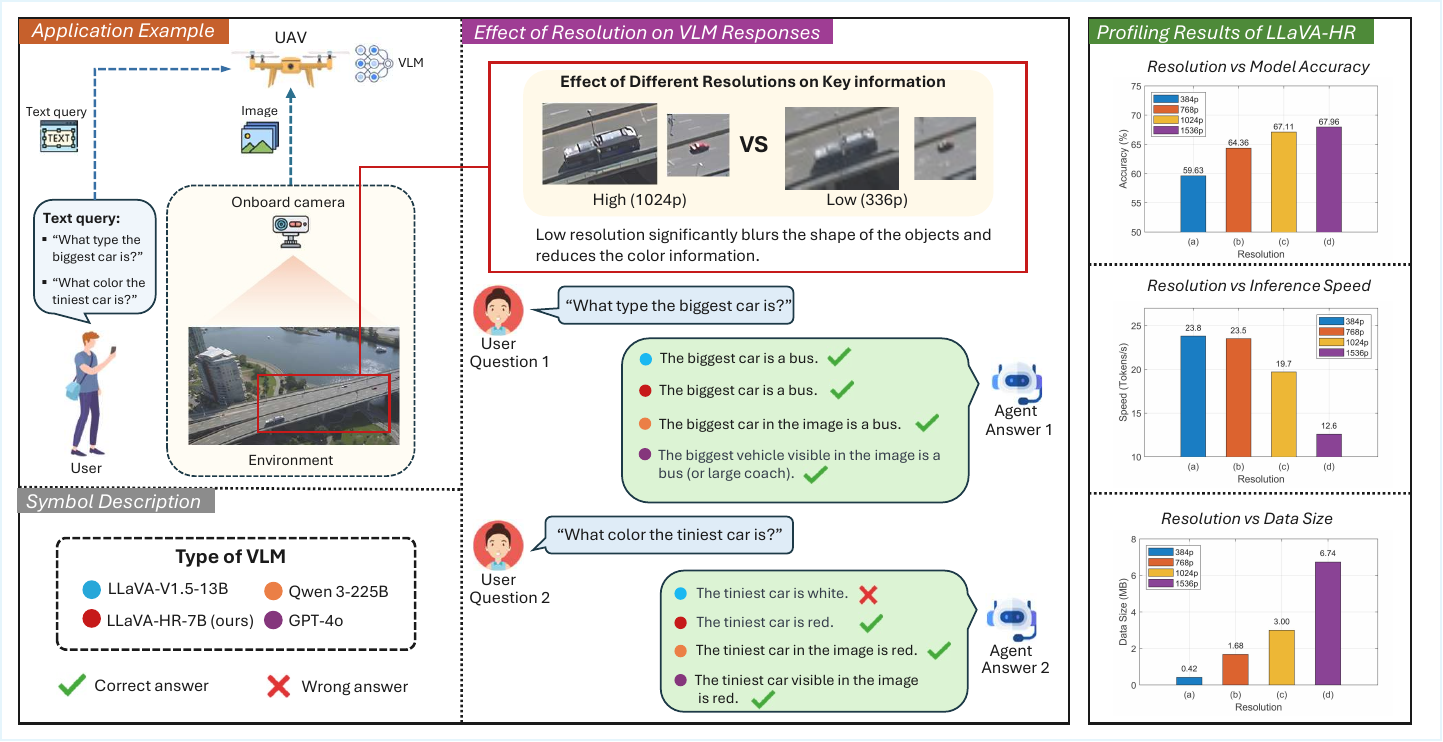}
    \caption{Impact of input resolution on TextVQA. 
Left and middle: our reproduced experiments using LLaVA-HR show that high-resolution inputs (e.g., $1024\mathrm{p}$) preserve fine details such as the tiny ``red'' car, enabling correct answers that are lost when downsampled to $336\mathrm{p}$. 
Notably, LLaVA-HR-7B achieves accuracy comparable to much larger models (Qwen3-225B, GPT-4o). 
Right: published profiling results of LLaVA-HR~\cite{luo2024feast}, illustrating the accuracy–efficiency–size trade-off across resolutions.}
    \label{fig:LLaVA}
\end{figure*}

First, we conduct a TextVQA case study to examine how input resolution influences the quality of answers produced by diverse VLMs, including LLaVA-1.5-13B, LLaVA-HR-7B, Qwen-3-225B~\cite{yang2025qwen3} and ChatGPT-4o. 
As shown in the left and middle panels of Fig.~\ref{fig:LLaVA}, the $1024\mathrm{p}$ image preserves fine-grained details such as the tiniest car's true color,
whereas downsampling to $336\mathrm{p}$ blurs small objects and reduces color information. 
This resolution gap directly affects the reasoning performance of VLMs: for \emph{Question 1} (``What type the biggest car is?''), all models consistently answer “bus,” since the large object remains identifiable at both resolutions.
However, for \textit{Question~2} (``What color the tiniest car is?''), only models that can process higher-resolution inputs are able to give the correct answer. 
LLaVA-1.5-13B defaulting to $336\mathrm{p}$ fails on this task, while LLaVA-HR-7B successfully outputs the correct answer. 

Notably, the comparison shows that LLaVA-HR-7B achieves answer quality comparable to larger models such as Qwen3-225B and GPT-4o, despite its significantly smaller size. 
This suggests that enabling high-resolution visual input offers a more efficient path to improving answer accuracy than simply scaling up model parameters. Given that our target application lies in UAV-enabled LAENet, where onboard storage and computational resources are inherently limited, we adopt LLaVA-HR as the reference VLM in our subsequent analysis. 

To complement the quantitative analysis, we further include the profiling results of LLaVA-HR~\cite{luo2024feast}, as shown in the right panel of Fig.~\ref{fig:LLaVA}. These results provide a system-level view of the resolution-accuracy–efficiency trade-off. 
Specifically, accuracy improves from 59.63\% at $384\mathrm{p}$ to 67.96\% at $1536\mathrm{p}$, 
while inference speed decreases from $23.8$ to $12.6$ tokens/s and input size grows from $0.42$ MB to $6.74$ MB. To support resolution-aware optimization in our system design, we define empirical lookup functions for both accuracy and inference speed. The resolution-dependent accuracy function $A_n(r_n)$ is modeled as:
\begin{align}
    A_n(r_n) =
    \begin{cases}
    59.63\%, & \text{if } r_n = 384\mathrm{p}, \\
    64.36\%, & \text{if } r_n = 768\mathrm{p}, \\
    67.11\%, & \text{if } r_n = 1024\mathrm{p}, \\
    67.96\%, & \text{if } r_n = 1536\mathrm{p}. \\
    \end{cases}
\end{align}
Similarly, we have the model inference speed function $v(r_n)$ reported in~\cite{luo2024feast} as follows:
\begin{align}
    v(r_n) =
    \begin{cases}
    23.8~\mathrm{tokens/s}, & \text{if } r_n = 384\mathrm{p}, \\
    19.9~\mathrm{tokens/s}, & \text{if } r_n = 768\mathrm{p}, \\
    19.7~\mathrm{tokens/s}, & \text{if } r_n = 1024\mathrm{p}, \\
    12.6~\mathrm{tokens/s}, & \text{if } r_n = 1536\mathrm{p}, \\
    \end{cases}
\end{align}
These empirical mappings serve as key inputs to our designed optimization framework, enabling resolution selection based on the desired trade-off between accuracy and efficiency.

\subsection{Comparative Analysis with Different Baselines}\label{subsec:comparative_analysis}

To comparatively analyze our ARPO-LLaRA framework, we introduce three baselines. The first is a \textit{Random Policy (RP)}, where the agent selects resolutions, velocities, and powers arbitrarily, serving as a benchmark without optimization. The second is \textit{ARPO-Geometric Heuristic (ARPO-GH)}, which applies ARPO for resolution and power allocation and then directs the UAV toward the geometric center of users~\cite{daugacsan2023resilient}. The third is \textit{ARPO-PPO}, which integrates ARPO with a PPO backbone~\cite{schulman2017proximal} using manually crafted reward functions aim to minimize the backlog of the bottleneck user. This comparison highlights the advantages of incorporating LLM-assisted reward design in ARPO-LLaRA.

\textbf{Convergence Performance.} 
Fig.~\ref{fig:latency_comparison} compares the convergence performance of our ARPO-LLaRA method against the baselines. The shaded area denotes the variance, and the solid line denotes the mean in terms of the latency performance. It can be observed that ARPO-LLaRA constantly outperforms all other baselines over the entire training episodes, which validates the superiority of our hierarchical optimization framework. Compared with RP and ARPO-GH, ARPO-LLaRA achieves significant improvements of around $45\%$ and $25\%$ in latency performance, respectively. 
Even against the robust ARPO-PPO baseline with a manually designed reward function, ARPO-LLaRA delivers a $13.7\%$ improvement in latency reduction, owing to the effectiveness of proposed LLM-augmented reward design. 
By comparing RP and ARPO-GH, we can also demonstrate the effectiveness of the proposed ARPO algorithm. Moreover, ARPO-LLaRA demonstrates faster convergence and greater training stability, thanks to the LLM-designed reward function that incorporates multiple informative terms. 
Specifically, the LLM-designed reward function penalizes tail backlog (aligning with worst-case latency), emphasizes overall throughput from unfinished users, and prioritizes users who currently dominate the overall completion time, thereby guiding toward a stable optimal policy more efficiently.




\begin{figure}
    \centering
    \includegraphics[width =.9\linewidth]{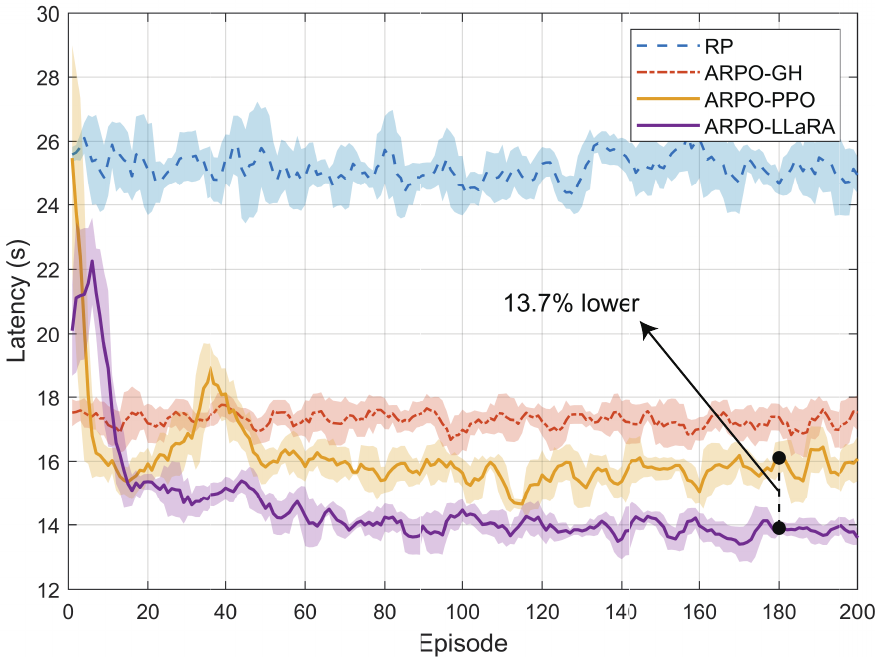}
    \caption{Convergence performance comparison between ARPO-LLaRA and different baselines.}
    \vspace{-10pt}
    \label{fig:latency_comparison}
\end{figure}

\textbf{Trajectory and Latency Performance.} To provide a clearer illustration of the method’s performance in deployment, Figs.~\ref{fig:trajectory_latency1} and \ref{fig:trajectory_latency2} compare UAV trajectories and latency under two user distributions comprising both high-demand and low-demand users.
Specifically, \textit{high-demand users} are defined as those issuing multiple queries ($M_n=2$) with stricter accuracy requirements ($A_n^{\min}=0.67$), whereas \textit{low-demand users} submit a single query ($M_n=1$) with a standard accuracy requirement ($A_n^{\min}=0.60$). Next, we detail our analysis on Figs.~\ref{fig:trajectory_latency1} and \ref{fig:trajectory_latency2}.

Fig.~\ref{fig:trajectory_latency1} examines performance under user distribution I, with two high-demand users located farther away in the northwest corner and two low-demand users positioned closer to the UAV start point. 
As shown in Fig.~\ref{fig:trajectory_latency1}(a), both ARPO-PPO and ARPO-LLaRA guide the UAV toward the high-demand users, since they dominate the system latency. Compared with ARPO-PPO, ARPO-LLaRA produces a smoother and more stable trajectory, avoiding unnecessary detours while still prioritizing bottleneck users. This suggests that the LLM-augmented reward design not only balances the UAV’s service priorities more effectively but also helps the agent converge toward a more efficient trajectory that minimizes latency with improved stability. The quantitative results in Fig.~\ref{fig:trajectory_latency1}(b) further confirm these observations. While RP suffers from the highest latency, ARPO-GH reduces latency by $27.4\%$. ARPO-PPO achieves an additional $10.8\%$ reduction, and ARPO-LLaRA yields the best performance with a further $11.7\%$ improvement over ARPO-PPO. These results highlight that the proposed ARPO method is effective in reducing latency, and the LLM-augmented reward design in ARPO-LLaRA provides a significant performance gain beyond manually designed rewards.

Fig.~\ref{fig:trajectory_latency2}(a) shows the UAV trajectory under user distribution II, where the high-demand users are located closer to the UAV start point and the low-demand users are farther away. Compared with distribution I, the UAV trajectories become shorter, as the bottleneck users can be reached more quickly. Both ARPO-PPO and ARPO-LLaRA still prioritize high-demand users, but ARPO-LLaRA maintains a smoother and more efficient path. The latency results in Fig.~\ref{fig:trajectory_latency2}(b) follow a similar trend as before. ARPO-GH achieves a notable latency reduction compared with RP, while ARPO-PPO brings further improvements. ARPO-LLaRA again achieves the lowest latency, with an additional $11.7\%$ gain over ARPO-PPO, demonstrating the consistent advantage of the LLM-augmented reward design.


\begin{figure}[!t]
    \centering
    \begin{subfigure}[b]{0.49\linewidth}
        \centering
\includegraphics[width=\linewidth]{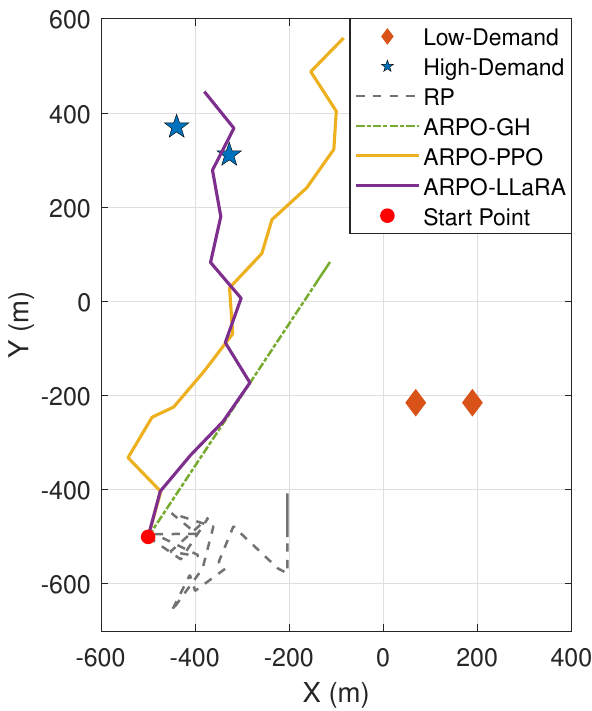}
    \caption{UAV horizontal trajectory.}
    \end{subfigure}
    \begin{subfigure}[b]{0.49\linewidth}
        \centering
\includegraphics[width=\linewidth]{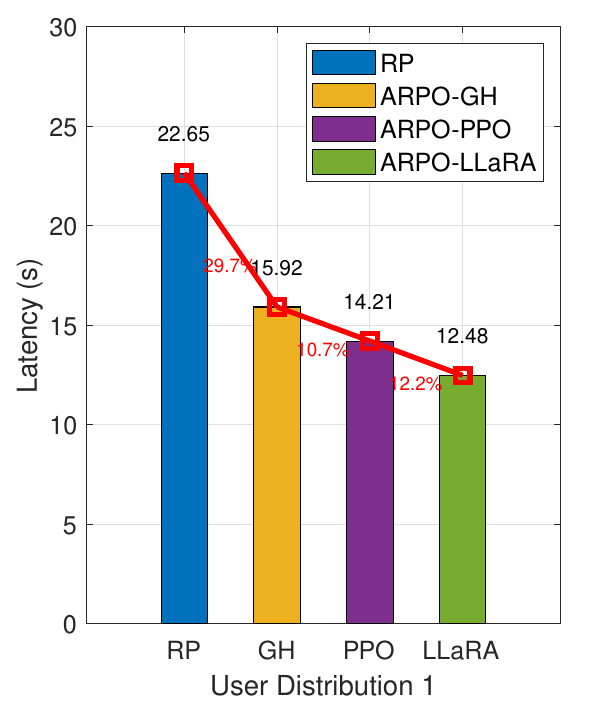}
    \caption{Achieved minimum latency.}
    \end{subfigure}
    \caption{Optimized UAV trajectory and minimum latency comparisons under user distribution I, 
    where high-demand users located farther and low-demand users closer.} 
    \label{fig:trajectory_latency1}
\end{figure}

\begin{figure}[!t]
    \centering
    \begin{subfigure}[b]{.49\linewidth}
        \centering
\includegraphics[width=\linewidth]{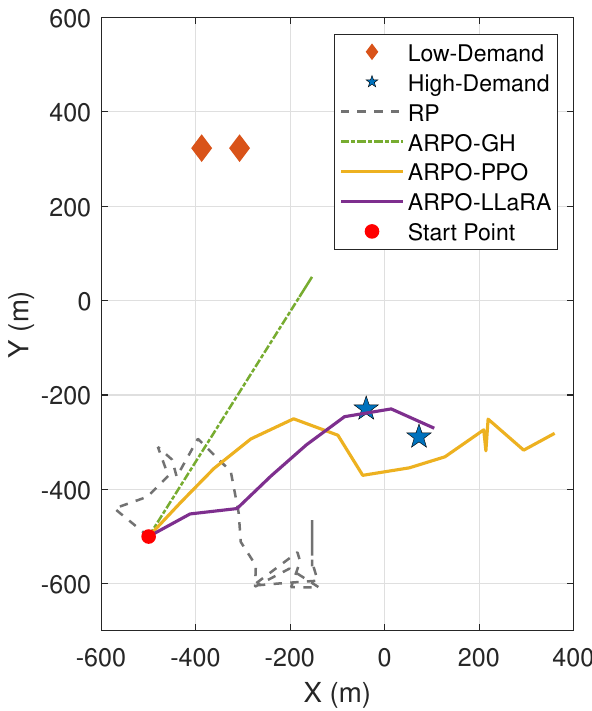}
    \caption{UAV horizontal trajectory.}
    \end{subfigure}
    \begin{subfigure}[b]{.49\linewidth}
        \centering
\includegraphics[width=\linewidth]{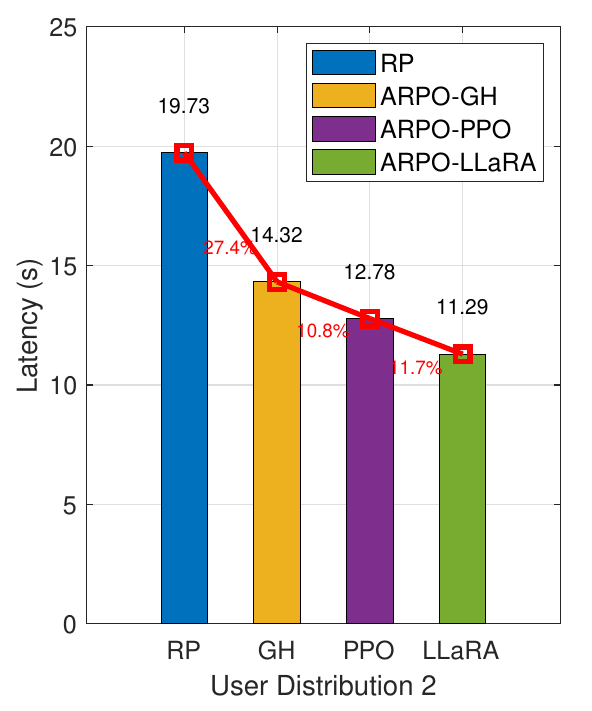}
    \caption{Achieved minimum latency.}
    \end{subfigure}
    \caption{Optimized UAV trajectory and minimum latency comparisons under user distribution II, 
    where high-demand users located closer and low-demand users farther.} 
    \label{fig:trajectory_latency2}
\end{figure}

\subsection{Impact of Weight Coefficients on Latency and Power}\label{subsec:trade_off}

We now examine how the weight coefficient influences the optimization result. Recall that in the objective function~\eqref{Original_Problem} of $\mathbb{P}_0$, the parameter $\zeta$ determines the relative importance assigned to power versus latency. When communication bandwidth is relatively abundant, or when the system is more sensitive to power consumption, increasing $\zeta$ encourages users to adopt lower transmit power.

We set the coefficient $\zeta$ within the range of $100$ to $1000$, while allocated bandwidth is set as $2$ MHz. The experimental result is reported in Fig.~\ref{fig:latency_power_zeta}. As illustrated, increasing $\zeta$ leads to a significant reduction in the total transmit power, which decreases from around $0.20$ W at $\zeta=100$ to below $0.07$ W at $\zeta=1000$. This trend evidently verifies that a larger weight on power effectively motivates users to transmit with smaller power levels. In contrast, latency exhibits the opposite trend: it grows from roughly $8.5$ s at $\zeta=100$ to nearly $10$ s at $\zeta=700$, after which it continues to increase more gradually. This occurs because reducing transmit power lowers the transmission rate, thereby increasing the overall latency. 


\begin{figure}
    \centering
    \includegraphics[width =.85\linewidth]{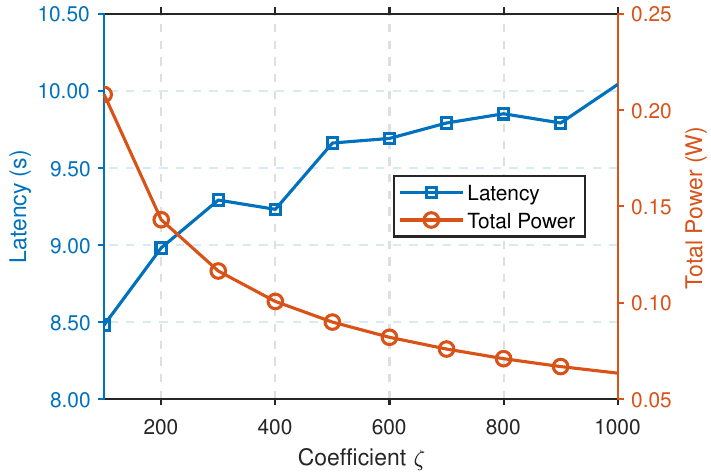}
    \caption{Overall latency and sum of user transmit powers under weight coefficient $\zeta$ with different values from $100$ to $1000$.}
    \label{fig:latency_power_zeta}
\end{figure}

\begin{figure*}[t]
    \centering
    \begin{subfigure}[b]{0.32\linewidth}
        \centering
\includegraphics[width=\linewidth]{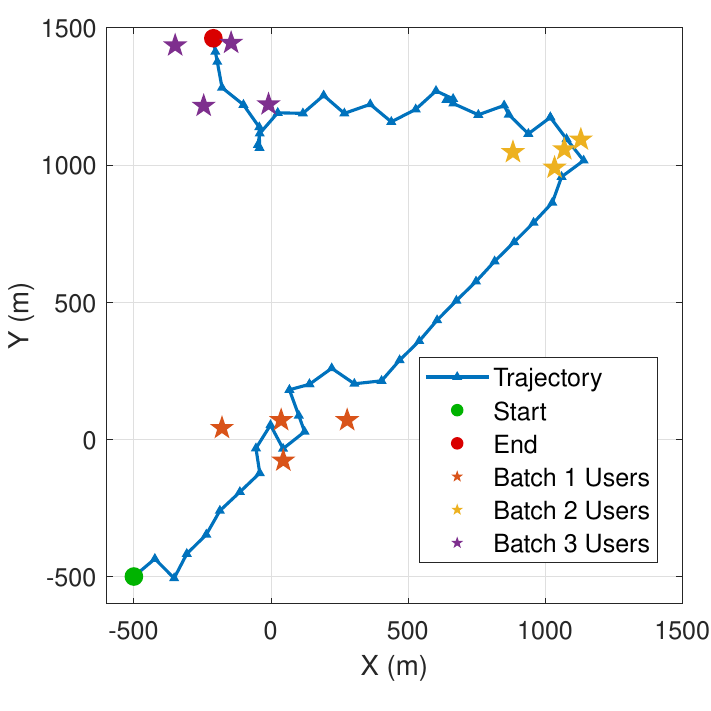}
    \caption{Top view of the UAV horizontal trajectory.}
    \end{subfigure}
    \hspace{50pt}
    \begin{subfigure}[b]{0.45\linewidth}
        \centering
\includegraphics[width=\linewidth]{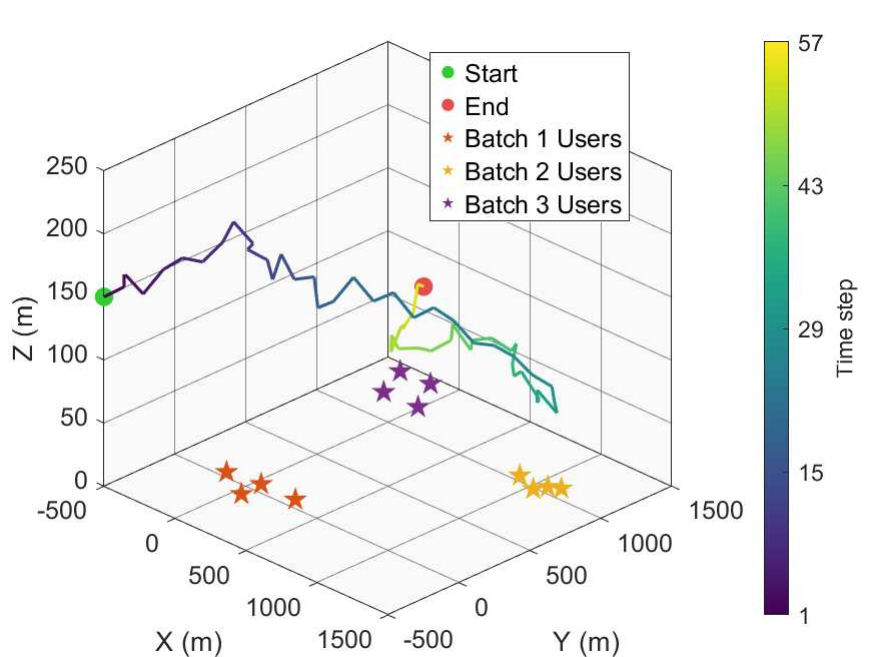}
    \caption{Visualization of the UAV 3D trajectory.}
    \end{subfigure}
    \caption{UAV trajectory performance for a multi-round service scenario. The UAV starts from the initial location, sequentially serves Batch 1, Batch 2, and Batch 3 users, and terminates at the final position.} 
    \label{fig:multi-batch}
\end{figure*}

\begin{figure}[t]
    \centering
    \includegraphics[width =.9\linewidth]{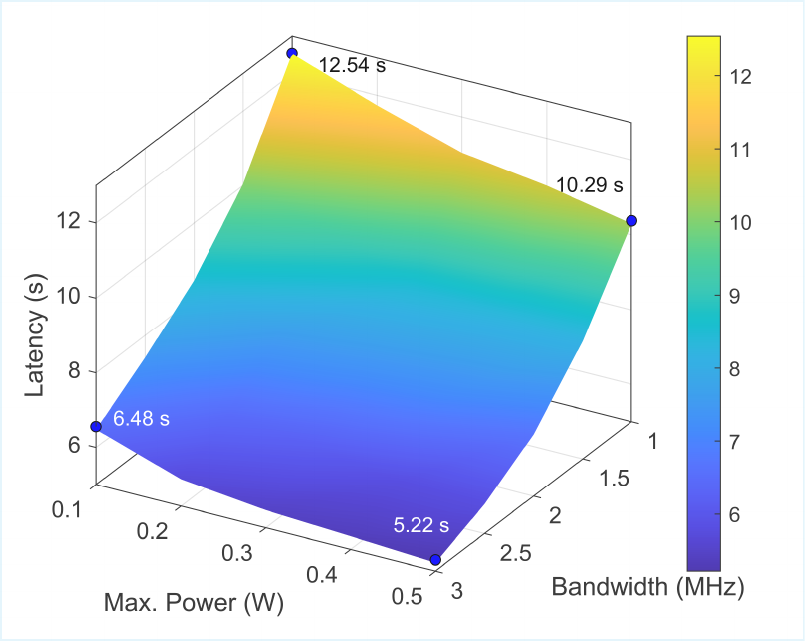}
    \caption{The overall latency versus different maximum transmit power and communication bandwidth.}
    \label{fig:latency_power_bandwidth}
\end{figure}

\subsection{Multi-round Optimization Performance for Multiple Batches}\label{subsec:multi-episode}

To further evaluate the adaptability of our proposed framework, we simulate a multi-round service scenario, where the UAV sequentially serves different batches of users. As shown in Fig.~\ref{fig:multi-batch}(a), the UAV first departs from the initial location (green dot) and completes service for Batch 1 users (orange stars). Afterwards, instead of returning to the initial point, it continues from the last position to serve Batch 2 users (yellow stars), and then proceeds to Batch 3 users (purple stars), eventually reaching the final destination (red dot). 

The 3D trajectory in Fig.~\ref{fig:multi-batch}(b) illustrates the UAV’s altitude variations and time progression, where the color bar indicates the time step. The results demonstrate that our proposed optimization framework can effectively plan the UAV trajectories across multiple service rounds, ensuring smooth transitions between spatially distributed users.

\subsection{Impact of Maximum Power and Bandwidth on Latency}\label{subsec:power_bandwidth}
To investigate how system resources affect the optimization performance, Fig.~\ref{fig:latency_power_bandwidth} shows the achieved latency under different maximum transmit power $P_n^{\max}$ and allocated bandwidth $B_n$. Intuitively, we can see that latency decreases as either bandwidth or transmit power increases, since both parameters directly enhance the achievable transmission rate. Specifically, bandwidth has a more dominant impact on the latency performance.
For instance, with a bandwidth of $3$ MHz and a maximum transmit power of $0.1$ W, the latency is reduced to $6.48$ s. In contrast, when bandwidth is constrained to $1$ MHz and transmit power is $0.5$ W, latency is reduced to $10.29$ s.

\section{Conclusion}\label{sec:conclusion}
In this paper, we have proposed a UAV-enabled LAENet that leverages VLMs for onboard inference services. Then, we have formulated a joint optimization problem for latency and power efficiency and introduced a hierarchical framework combining ARPO and LLaRA to optimize resolutions, transmit powers, and UAV trajectories. Simulations have shown that our approach reduces latency while meeting accuracy requirements and scales well under diverse resource and service settings, highlighting its potential for practical inference-as-a-service in LAENets.



\bibliographystyle{IEEEtran}
\bibliography{ref}

\appendices

\ifCLASSOPTIONcaptionsoff
  \newpage
\fi

\end{document}